\documentclass[sigconf]{acmart}
\usepackage{soul}
\usepackage{url}
\usepackage{hyperref}
\usepackage[utf8]{inputenc}
\usepackage{caption}
\usepackage{xcolor}
\usepackage{colortbl}
\usepackage{graphicx}
\usepackage{caption}
\usepackage{courier}
\usepackage{multirow}
\usepackage{subcaption}
\usepackage{color}
\usepackage{amssymb}
\usepackage{gensymb}
\usepackage{enumitem}
\usepackage{booktabs}
\usepackage{dsfont}
\usepackage[linesnumbered, ruled]{algorithm2e}
\usepackage{amsmath}
\usepackage[symbol]{footmisc}

\copyrightyear{2020}
\acmYear{2020}
\setcopyright{acmlicensed}\acmConference[CIKM '20]{Proceedings of the 29th ACM International Conference on Information and Knowledge Management}{October 19--23, 2020}{Virtual Event, Ireland}
\acmBooktitle{Proceedings of the 29th ACM International Conference on Information and Knowledge Management (CIKM '20), October 19--23, 2020, Virtual Event, Ireland}
\acmPrice{15.00}
\acmDOI{10.1145/3340531.3411872}
\acmISBN{978-1-4503-6859-9/20/10}

\newcommand{\ours}{{SL-DSGCN}\xspace}
\newcommand{\dsgnn}{{DSGCN}\xspace}
\def \E {\mathcal{E}}
\def \G {\mathcal{G}}
\def \V {\mathcal{V}}
\def \N {\mathcal{N}}

\def \L {\mathcal{L}}

\def \R {\mathbb{R}}

\def \X {\mathbf{X}}

\def \W {\mathbf{W}}

\def \x {\mathbf{x}}

\title[Investigating and Mitigating Degree-Related Biases in Graph Convolutional Networks]{Investigating and Mitigating Degree-Related Biases in \\ Graph Convolutional Networks}

\author{Xianfeng Tang$^\dagger$, Huaxiu Yao$^\dagger$, Yiwei Sun$^\dagger$, Yiqi Wang$^\ddagger$, Jiliang Tang$^\ddagger$} \author{Charu Aggarwal$^\mathsection$, Prasenjit Mitra$^\dagger$, Suhang Wang$^{\dagger*}$}
\affiliation{
\institution{The Pennsylvania State University$^\dagger$, Michigan State University$^\ddagger$,
IBM T.J. Watson, NY, USA$^\mathsection$}
}
\affiliation{
  \institution{\{xut10,huy144,yus162,pum10, szw494\}@psu.edu  \{wangy206,tangjili\}@msu.edu  charu@us.ibm.com}
}
\renewcommand{\shortauthors}{Tang et al.}

\begin{document}
\begin{abstract}
    
    Graph Convolutional Networks (GCNs) show promising results for semi-supervised learning tasks on graphs, thus become favorable comparing with other approaches. Despite the remarkable success of GCNs, it is difficult to train GCNs with insufficient supervision. When labeled data are limited, the performance of GCNs becomes unsatisfying for low-degree nodes.
    While some prior work analyze successes and failures of GCNs on the entire model level, profiling GCNs on individual node level is still underexplored.
    
    In this paper, we analyze GCNs in regard to the node degree distribution. From empirical observation to theoretical proof, we confirm that GCNs are biased towards nodes with larger degrees with higher accuracy on them, even if high-degree nodes are underrepresented in most graphs.
    We further develop a novel Self-Supervised-Learning Degree-Specific GCN (\ours) that mitigate the degree-related biases of GCNs from model and data aspects.
    Firstly, we propose a degree-specific GCN layer that captures both discrepancies and similarities of nodes with different degrees, which reduces the inner model-aspect biases of GCNs caused by sharing the same parameters with all nodes. 
    Secondly, we design a self-supervised-learning algorithm that creates pseudo labels with uncertainty scores on unlabeled nodes with a Bayesian neural network. 
    Pseudo labels increase the chance of connecting to labeled neighbors for low-degree nodes, thus reducing the biases of GCNs from the data perspective.
    Uncertainty scores are further exploited to weight pseudo labels dynamically in the stochastic gradient descent for \ours.
    Experiments on three benchmark datasets show \ours not only outperforms state-of-the-art self-training/self-supervised-learning GCN methods, but also improves GCN accuracy dramatically for low-degree nodes.

\end{abstract}
\maketitle
\section{Introduction} 

Over last few years, Graph Convolutional Networks (GCNs) have benefited many real world applications across different domains, such as molecule design~\cite{you2018graph}, financial fraud detection~\cite{wang2019semi}, traffic prediction~\cite{wang2020traffic,yu2017spatio}, and user behavior analysis~\cite{tang2020knowing,li2020few,huang2019online}. One of the most important and challenging applications for GCNs is to classify nodes in a semi-supervised manner. In semi-supervised learning, GCNs recursively update the feature representation of each node by applying node-agnostic transformation parameters. The whole training process is supervised by a few labeled nodes. \footnote{Suhang Wang is the corresponding author.}

However, degree distributions of most real-world graphs (e.g., citation graphs, review graphs, etc.) are power-law \cite{albert2002statistical,faloutsos1999power,clauset2009power}. While the degree of major nodes are relatively small, few nodes on the long-tail side can dominate the training/learning of GCNs (we refer to Figure~\ref{fig:degree_dist} in the analysis section as examples). We argue the power-law distributed node degree could hurt the performance of GCNs.
On the one hand, nodes on such a graph are not independent and identically distributed (\textit{i.i.d}), thus the parameters of a GCN should not be shared by all nodes. 
As suggested by~\cite{mislove2007measurement}, nodes with various degrees play different roles in the graph. Taking social networks as an example, high-degree nodes are usually leaders with higher influence; while most low-degree ones are at the fringes of the network. Current GCNs with node-agnostic parameters overlook the complex relations and roles of nodes with different degrees.
On the other hand, the non-\textit{i.i.d} node degrees can hurt the message-passing mechanism of GCNs.
In fact, the superior performance of GCNs relies on the information propagating from labeled nodes to unlabeled nodes \cite{hamilton2017inductive}. Obviously, nodes with lower degrees are less likely to be connected to labeled neighbors, compared with high-degree ones.
As a result, less information are passed to these low-degree nodes, resulting in unsatisfying or even poor prediction performance. 
Few literature have explored the effects of non-\textit{i.i.d} node degrees on real-world graphs. Recently, \citeauthor{wu2019demo} \cite{wu2019demo} propose a multi-task learning framework for GCNs, where the degree information is encoded into learned node representations. However, simply incorporating the value of degree as an extra feature does not solve the potential biases of GCNs, and low-degree nodes still suffer from the insufficient supervisions.

Therefore, in this paper, we analyze the degree-related biases in GCNs thoroughly.
First, we design a series of observational tests to validate our assumption: the performance of GCNs are not evenly distributed regarding node degrees, and GCNs are biased on low-degree nodes.
We further prove that the training of GCNs are more sensitive to nodes with higher degrees using sensitivity analysis and influence functions in statistics \cite{koh2017understanding,xu2018representation}. Inspired by the analytic results, we realize two challenges of addressing the degree-related biases in GCNs as follows:
\textbf{(C1) How to capture the complex relation among nodes with different degrees?} We recognize three types of node relations including global shared relation, local intra-relation, and local inter-relation. The global shared relation captures the common property among all nodes in the whole graph (i.e., what GCNs already done); the local intra-relation describes the similarity of nodes with the same degree; and the local inter-relation further characterizes the interacted information from nodes with similar degrees, as they may behave likewise. Therefore, a sufficiently generalized and powerful degree-specific GCN is required, which not only balances the global generalization and local degree customization of different nodes, but also captures local relation among nodes with various degrees; and
\textbf{(C2) How to provide effective and robust information to facilitate the learning of GCNs on low-degree nodes?} 
It is non-trivial to make accurate predictions with limited labeled neighbors, due to the biased information propagation. How to create sufficient supervisions for low-degree nodes is extremely challenging.

To address these challenges, in this paper, we propose a novel Self-Supervised-Learning Degree-Specific GCN (\ours), which reduces the biases from non-\textit{i.i.d} node degrees in conventional GCNs. In particular, we first design a degree-specific GNN layer, which considers both globally shared information and local relation among nodes with same degree value. A recurrent neural network (RNN) based parameter generator is designed for modeling the inter-degree relation, which is ignored in the prior work DEMO-Net \cite{wu2019demo}.
We then leverage the massive unlabeled nodes to construct artificial supervisions for low-degree nodes. We propose a self-supervised-learning paradigm where a Bayesian neural network serves as the teacher and assigns pseudo/soft labels jointly with uncertainty scores on unlabeled nodes. We further utilize the uncertainty scores as a guidance in stochastic gradient descent to prevent overfitting inaccurate pseudo labels when training \ours.
\ours is evaluated on three benchmark datasets and show superior performance over state-of-the-art methods. Besides, it reduces label prediction error on low-degree nodes dramatically.

In summary, our contributions are three-fold:
\begin{itemize}[leftmargin=*]
    \item We study a novel problem of addressing the degree-related biases in GCNs. To the best of our knowledge, we are the first to analyze this problem empirically and theoretically.
    \item We design \ours that tackles the degree-related biases in GCNs from both model and data distribution aspects using the proposed degree-specific GCN layer and self-learning algorithm, correspondingly.
    \item We validate \ours on three benchmark graph datasets and confirm that \ours not only out-performs state-of-the-art baselines, but also improves the prediction accuracy on low-degree nodes significantly.
\end{itemize}
\section{Related Work}
In this section, we review related works, which includes graph neural networks and self-supervised learning.

\subsection{Graph Convolutional Neural Networks}
Graph data are ubiquitous in real-world. Recently, graph convolutional neural networks (GCNNs) have achieved state-of-the-art performance for many graph mining tasks \cite{xu2018powerful,kipf2016semi,hamilton2017inductive}, and many efforts have been taken~\cite{wu2020comprehensive,ying2018graph,xu2018representation,jin2020graph,tang2020transferring,sun2020adversarial}.
In general, these GCNNs could be divided into two categorizes: spectral based GCNNs and spatial-based GCNNs. \citeauthor{bruna2013spectral} firstly propose the spectral based GCNNs~\cite{bruna2013spectral} by applying the spectral filter on the local spectral space according to the spectral graph theory. Following this work, various spectral-based GCNNs~\cite{bronstein2017geometric, defferrard2016convolutional,hamilton2017inductive, kipf2016semi} are developed to improve the performances. GCN~\cite{kipf2016semi} aggregates the neighborhood information from the perspective of spectral theory. 
With the similar intuition, GraphSAGE~\cite{hamilton2017inductive} extends prior works in the inductive setting.
The spectral based GCNNs usually require to compute the Laplacian eigenvectors or the approximated eigenvalues as suggested by spectral theory, and these methods are inefficient on large scale graph. Different from the  spectral based ones, to improve the efficiency, the spatial-based GCNNs~\cite{atwood2016diffusion, velivckovic2017graph, zhang2018gaan} attempt to directly capture the spatial topological information and use the mini-batch training schema. For example, DCNN\cite{atwood2016diffusion} combines graph convolutional operator with the diffusion process and
\citeauthor{velivckovic2017graph} proposes the graph attention network~\cite{velivckovic2017graph} with the self-attention mechanism on the neighbors of the node and assign different weights accordingly during the aggregation process.
Of all these GCNNs, GNNs \cite{kipf2016semi} are highly favorable by the computer science community \cite{sun2019multi,li2018deeper} due to the reliable performance. Thus, we select GCNs for this work.

Though GCNs have show promising results, recent advancements~\cite{zugner2018adversarial, dai2018adversarial, xu2018powerful} also reveal various issues of GCNs including the over-smoothing and the vulnerability. In this paper, we empirically validate a new issue of GCNNs, i.e., \textit{GCNNs are biased towards high-degree nodes and have low accuracy on low-degree ones}. A potential reason is the imbalanced labeled node distribution. The issue is amplified when the total amount of labeled node for training is small. 


\subsection{Self-Supervised Learning}
Recently, self-supervised learning, which generally refers to explicitly training models with automatically generated labels, has become a successful approach in computer vision and natural language processing for unsupervised pretraining and for addressing the issue of lacking labeled data~\cite{jing2020self}. For example, pretext tasks such as Image
Inpainting~\cite{pathak2016context} and Image Jigsaw
Puzzle~\cite{noroozi2016unsupervised} are widely adopted in computer vision domains.

The success of self-supervision has motivated its study in graph mining domains. Though still in its early stage, there are a few seminal work trying to exploit self-supervised training to improve the performance of GCNs \cite{jin2020self}. 
For example, \citeauthor{li2018deeper}~\cite{li2018deeper} propose the co-training and self-training based GCN models by expanding the training node set with pseudo labels from its nearest neighborhoods; \citeauthor{sun2019multi} \cite{sun2019multi} combine DeepCluster \cite{caron2018deep} with a multi-stage training framework so that the generalization performance of GCNs with few labeled nodes are improved. 

Despite their initial success, existing studies mainly utilize self-supervised training as a trick for GCNs, without digging deep into why self-supervised training can improve the performance and what kind of nodes are benefited most from the self-supervised training. Our work is inherently different from existing ones on self-supervised GCNs. The lack of labeled neighborhoods among low-degree nodes motivate us to explore self-supervised training to balance the label distribution. The proposed self-supervision based one teacher-student network is also different from existing work. In addition, we also address the issue from the perspective of degree-specific layers. 

To the best of our knowledge, only few work address the degree non-\textit{i.i.d} sampled problem. DEMO-Net~\cite{wu2019demo} learn the degree-specific representation for each node via the explicitly designed hash table. This work is significantly different from ours. Besides, \textit{it fails to capture the similarity of nodes with close degree values}, where the RNN-based parameter generator in \ours is able to do so.


\section{Preliminaries}
We use $\G = (\V, \E, \X)$ to denote a graph, where $\V = \{v_1, \dots, v_N\}$ is the set of $N$ nodes, $\E \subseteq \V \times \V$ represents the set of edges, and $\X = \{\x_1, \dots, \x_N\}$ indicates node features.
We use $d_i \in \R^+$ to denote the degree of node $v_i$. 
In semi-supervised setting, partial nodes come with labels and are denoted as $\V^L$, where the corresponding label of node $v_i$ is $y_i$. Similarly, the unlabeled part is defined as $\V^U$.

We introduce the architecture of a GCN. A GCN contains multiple layers. Each layer transforms its input node features to another Euclidean space as output. Different from fully-connected layers, a GCN layer takes first-order neighbors' information into consideration when transforming the feature vector of a node. This ``message-passing'' mechanism ensures the initial features of any two nodes can affect each other even if they are faraway neighbors, along with the network going deeper.
We use $\x_v^{l}$ to denote the learned representation of node $v$ from the $l$-th layer in a GNN ($l = 1,\cdots,L$). Specifically, $\x_v^{0} = \x_v$.
The output node features of the $l$-th layer, which also formulate the input to the next layer, are generated as follows:
\begin{equation}
\label{eqn:general_layer}
    \x_i^{l+1} =\sigma\Big(\sum_{v_j \in \N(i)} \frac{1}{\sqrt{d_i\cdot d_j}} \W^l \x_j^l\Big),
\end{equation}
where $\N(i)$ denotes the immediate neighbor nodes of $v_i$ and  $\sigma$ is the activation function (e.g., ReLU). 

We take node classification as an example task for the rest of the paper, without loss of generality. The objective of training GNNs is to minimize the following cross-entropy loss function:
\begin{equation} \label{eqn:loss}
    \L = \sum_{v_i \in \V^L}\L(v_i) = -\sum_{v_i \in \V^L} y_i\log \hat{y}_i,
\end{equation}
where $y_v$ and $\hat{y}_v$ are true and predicted labels, respectively. Typically, $\hat{y}_v = \text{Softmax}(\x_v^{(L)})$ is acquired by applying Softmax to the representations from the last layer.
\section{Data Analysis} 
In this section, we conduct preliminary analysis on real-world graphs to show the properties of real-world graphs for semi-supervised node classification and the issue of GCNs on these datasets. The preliminary analysis lays a solid foundation and paves us a way to design better GCNs. Since we aim to discover the issue of GCNs on real-world datasets, we choose four widely used datasets from GCNs literature to perform the analysis, which includes Cora, Citeseer, Pubmed \cite{kipf2016semi}, and Reddit \cite{hamilton2017inductive}. Note that the split of training, validation and testing on all datasets are the same as described in the cited papers. 

\subsection{Degree Distribution}
The degree distribution of most real-world graphs follows the power-law~\cite{faloutsos1999power,albert2002statistical}. To verify this, we plot the degree distribution of the four datasets in Figure \ref{fig:degree_dist}. As we can see from the figure, degrees of the majority of nodes are relatively low, and decrease as the value of degree raise up. The shape of the degree distributions verify our assumption. The power-law distribution indicates nodes on graph are non-\textit{i.i.d} distributed. Applying the same network parameters on all nodes may result in sub-optimal prediction/classification.
\begin{figure}[t]
    \centering
    \includegraphics[width=0.45\columnwidth]{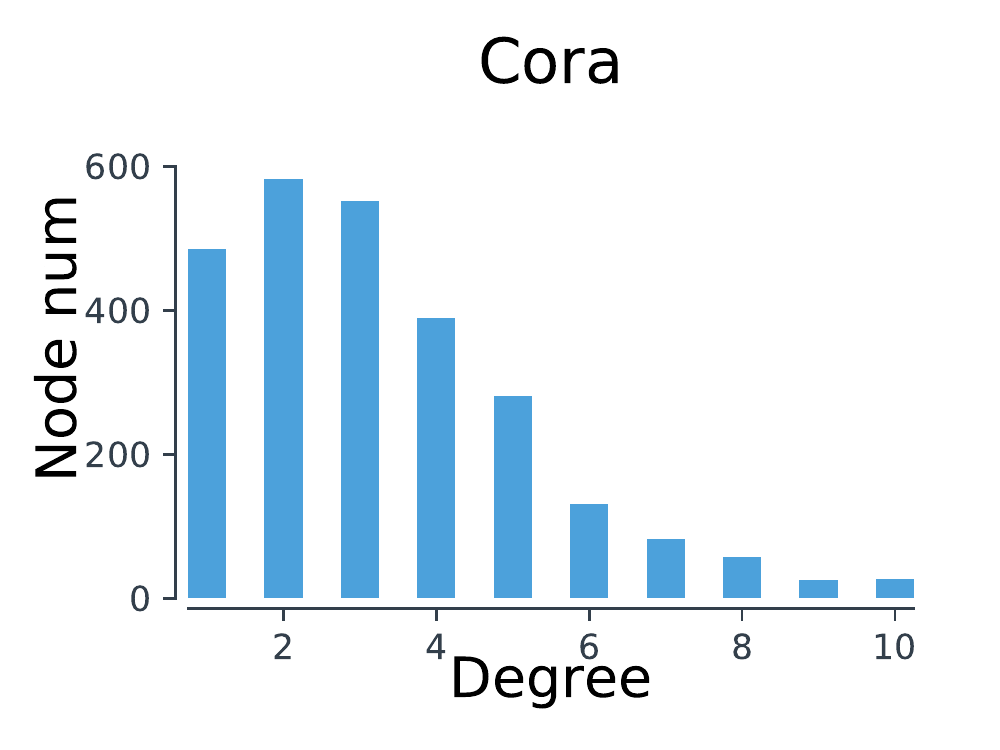} ~~
    \includegraphics[width=0.45\columnwidth]{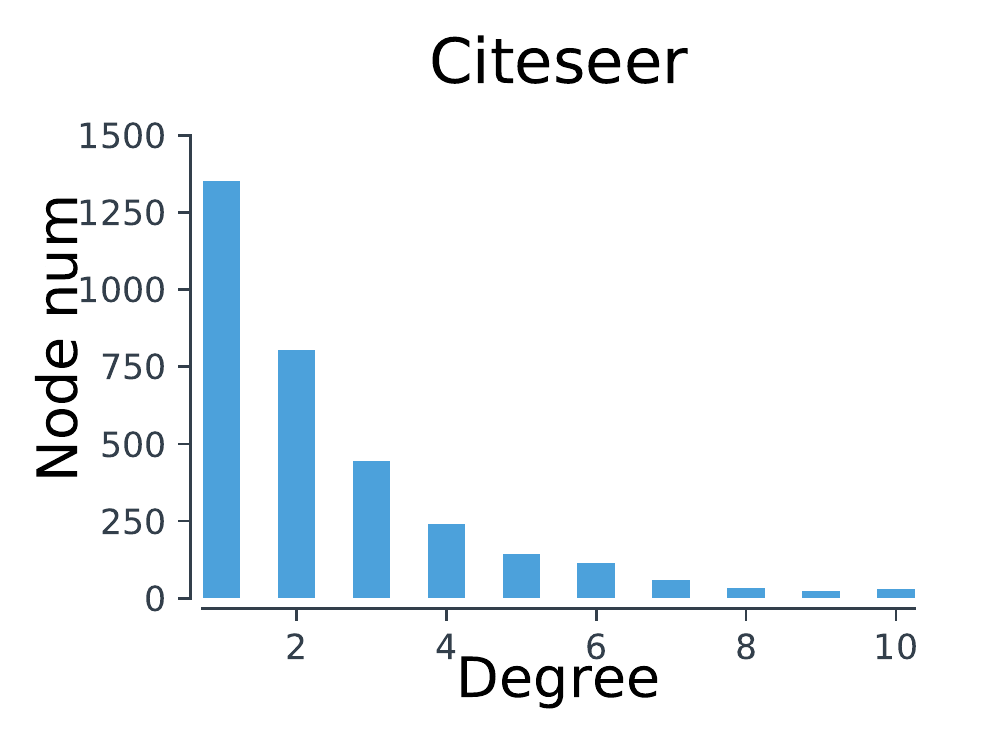}
    \hfill
    \includegraphics[width=0.45\columnwidth]{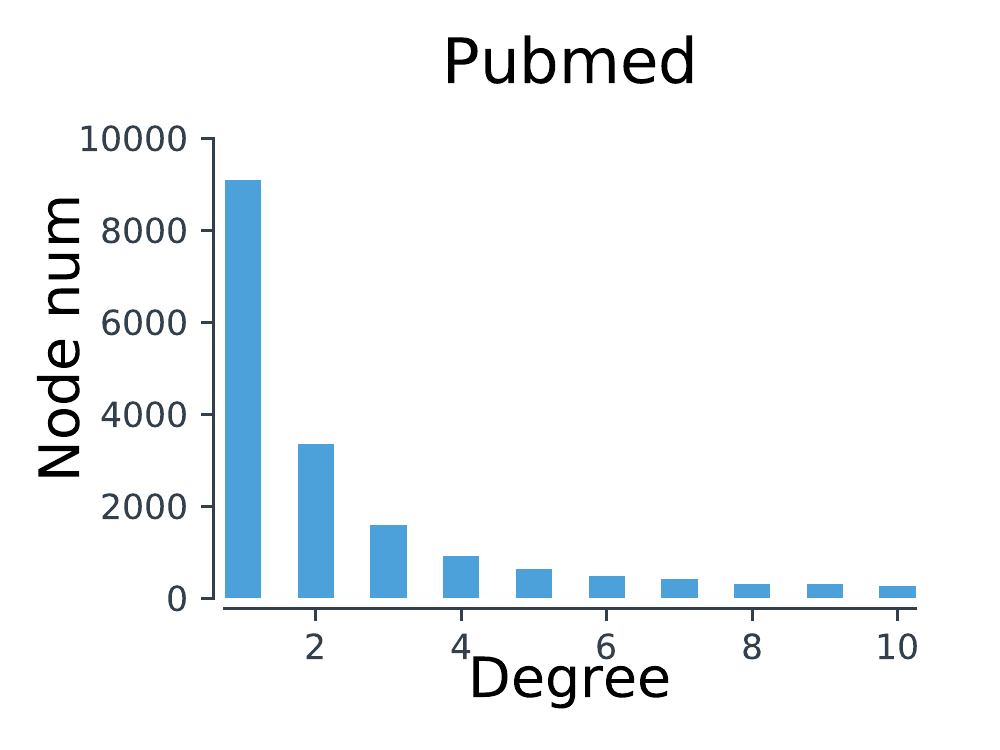} ~~
    \includegraphics[width=0.45\columnwidth]{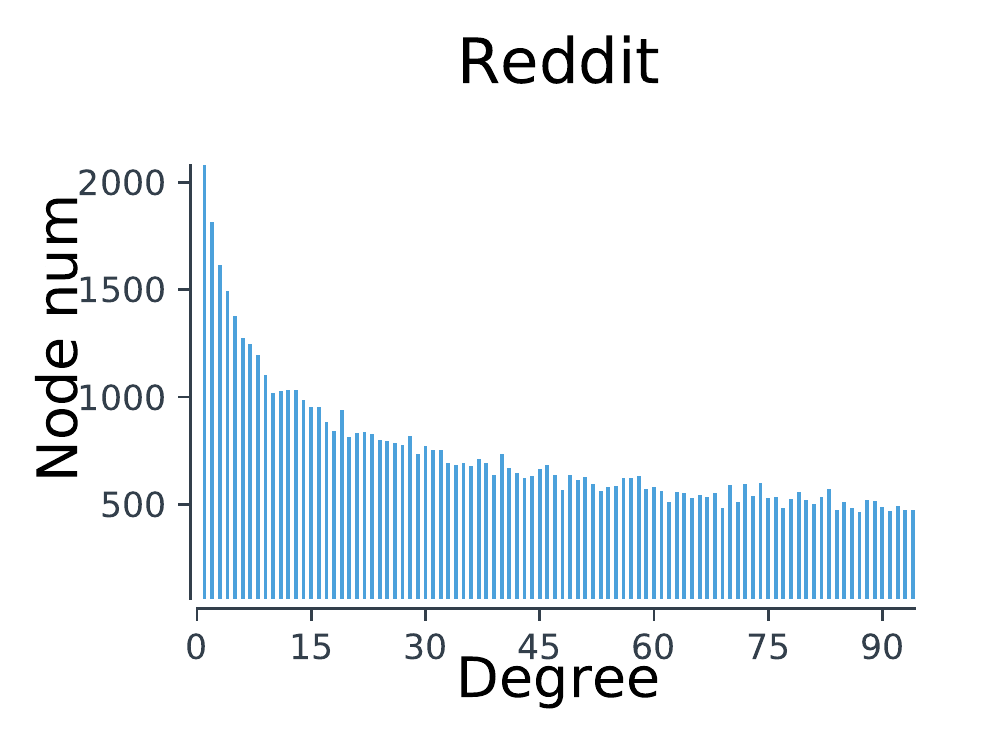}
    \vskip -1em
    \caption{Degree distribution.}
    \label{fig:degree_dist}
    \vskip -1em
\end{figure}

\subsection{Accuracy Varying Node Degree}
GCNs rely on message-passing mechanism, and aggregates the information from neighbors to learn representative embedding vectors. Because the degree of nodes follows a nonuniform (power-law) distribution, low-degree nodes, which are the majority, will receive less information during the aggregation. As a results, the error rate on low-degree nodes could be higher. To validate the assumption, we train GCNs following the same setting in \cite{kipf2016semi}, and report its error rate on node classification tasks w.r.t degree of nodes. From Figure \ref{fig:error_dist}, we find that, when degree is small, the error rate decreases significantly as the degree of nodes becomes larger. This verify our assumption that low-degree nodes receive less information during the aggregation and GCNs is biased against low-degree nodes.
\begin{figure}[!h]
    \centering
    \includegraphics[width=0.45\columnwidth]{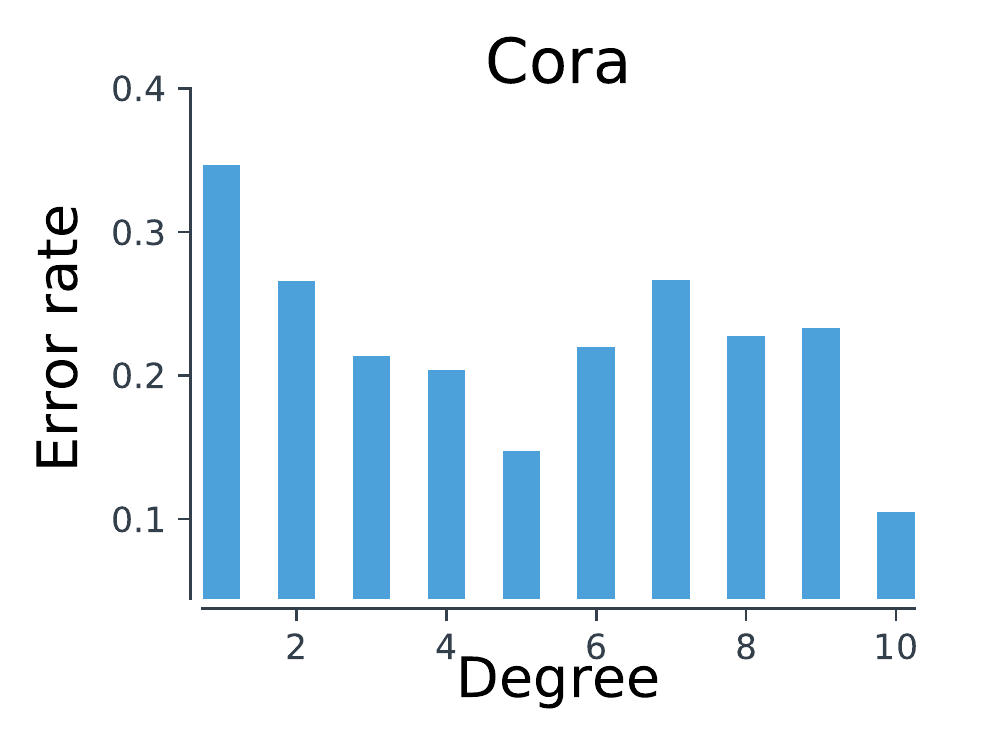} ~~
    \includegraphics[width=0.45\columnwidth]{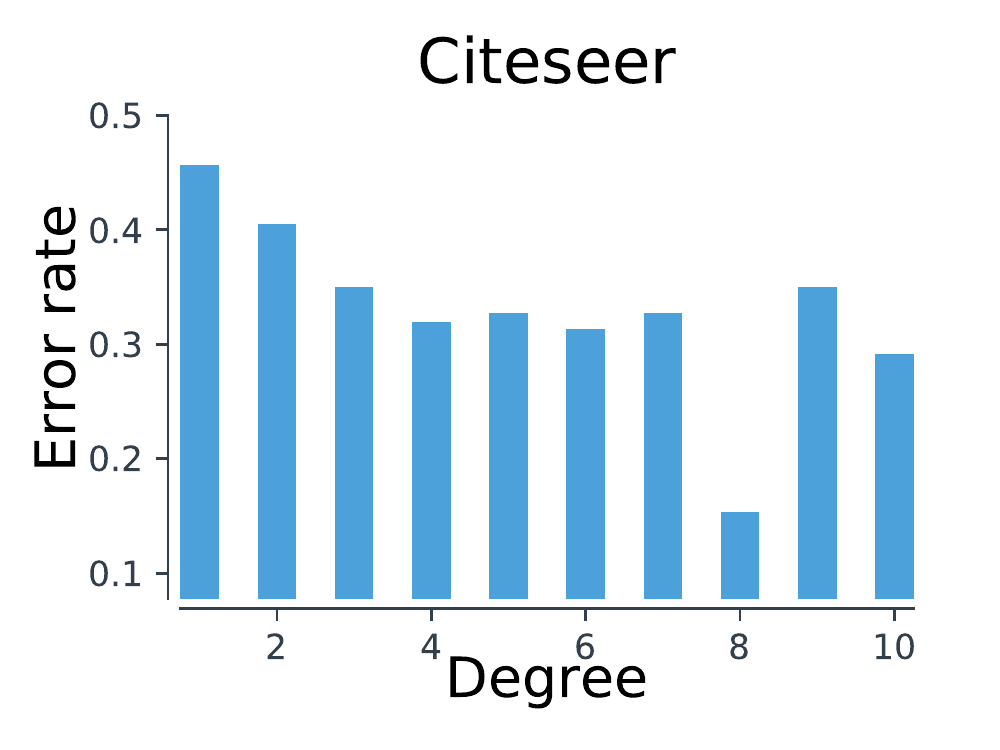}
    \hfill
    \includegraphics[width=0.45\columnwidth]{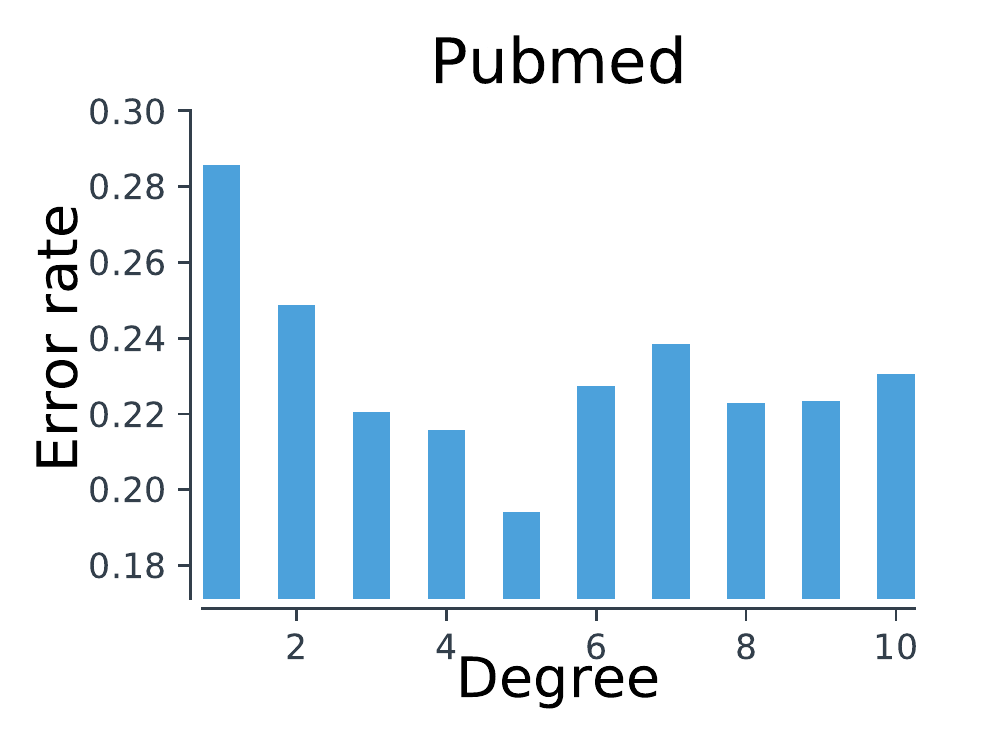} ~~
    \includegraphics[width=0.45\columnwidth]{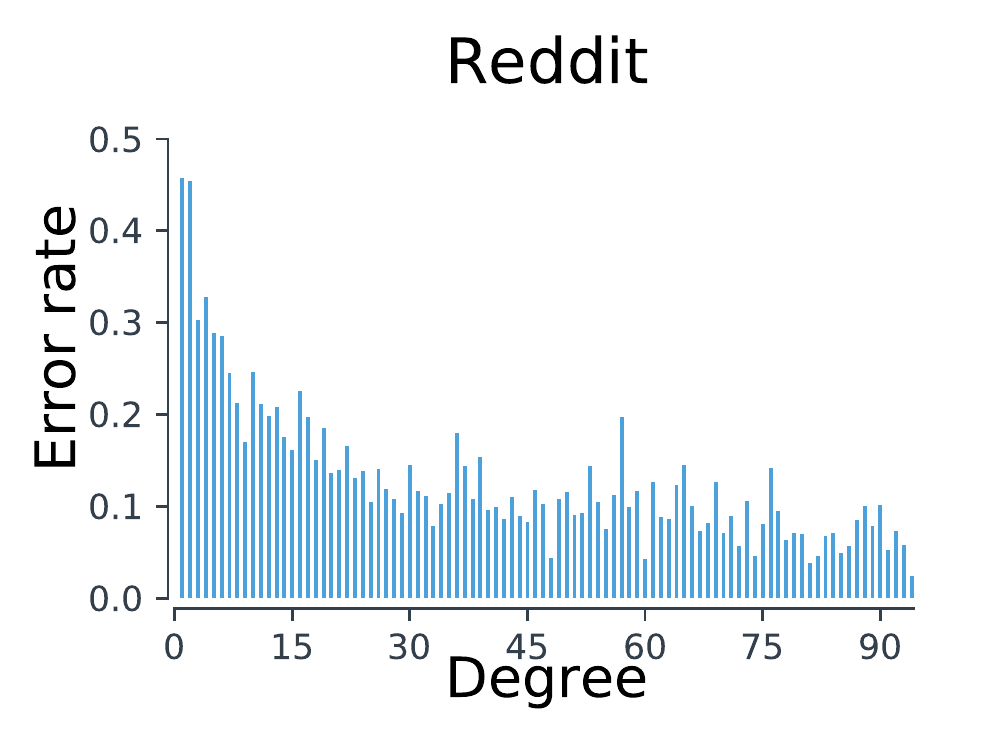}
    \vskip -1em
    \caption{Error distribution w.r.t node degree.}
    \label{fig:error_dist}
    \vskip -1.2em
\end{figure}

\subsection{Labeled Neighbor Distribution}
\begin{figure}[t]
    \centering
    \includegraphics[width=0.45\columnwidth]{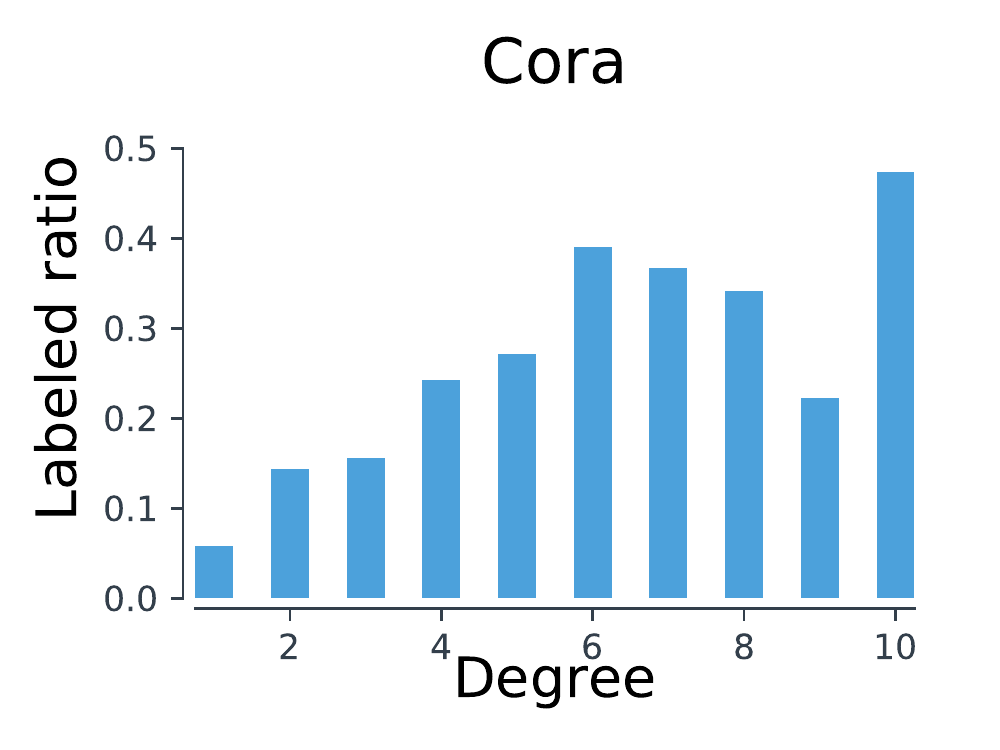} ~~
    \includegraphics[width=0.45\columnwidth]{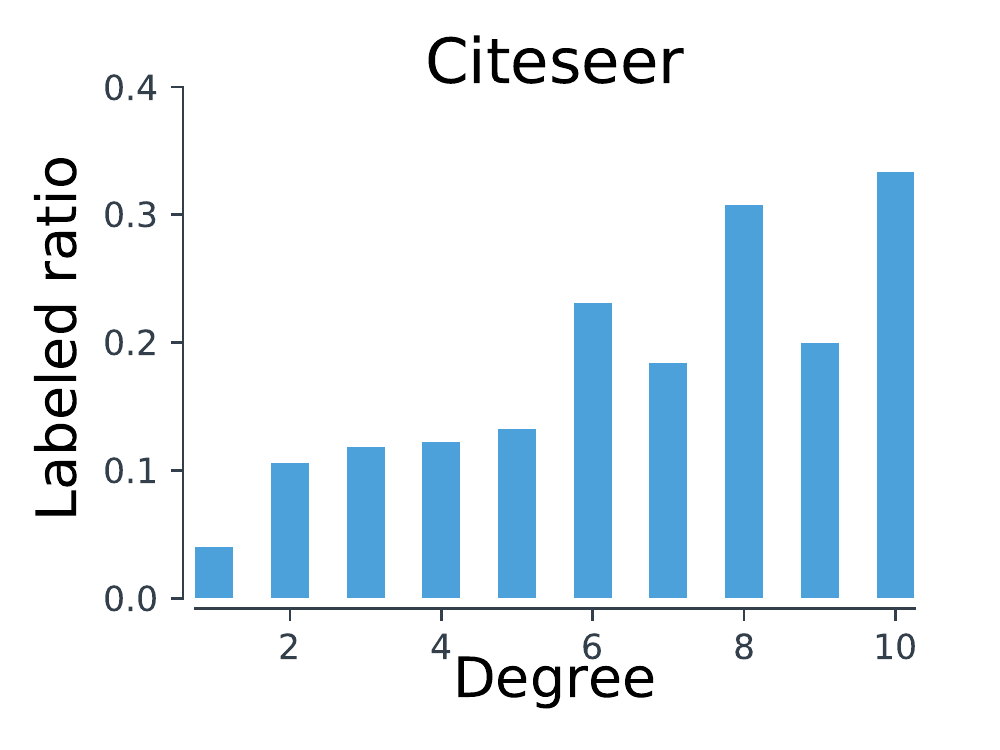}
    \hfill
    \includegraphics[width=0.45\columnwidth]{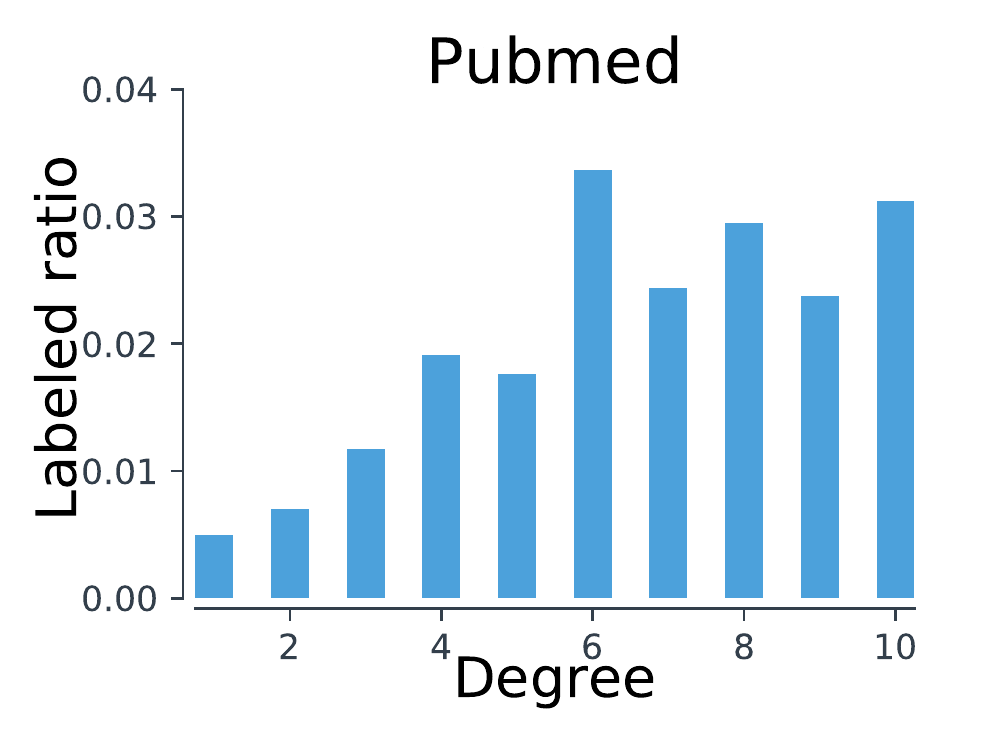} ~~
    \includegraphics[width=0.45\columnwidth]{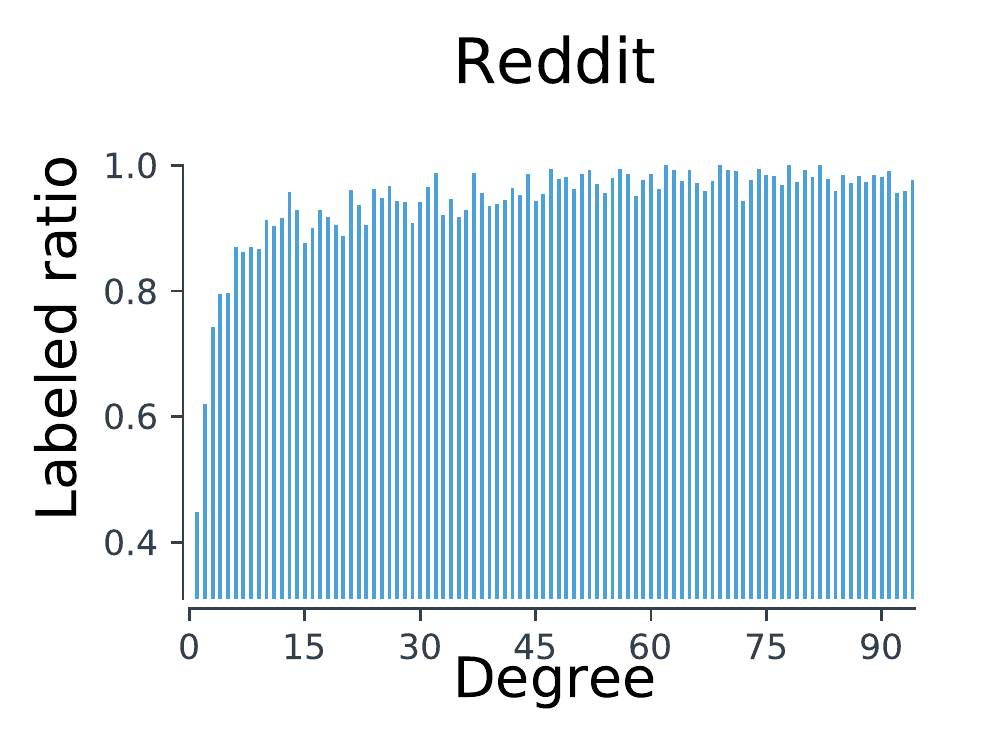}
    \vskip -1em
    \caption{Ratio of being neighbor with a labeled node.}
    \label{fig:has_labeled_nbhd_dist}
    \vskip -1.2em
\end{figure}

To further understand how the non-uniform degree distribution hurts GCNs, we analyze the probability of being connected to any labeled neighbor w.r.t node degree, as illustrated in
Figure \ref{fig:has_labeled_nbhd_dist}. We can conclude that nodes with higher degrees are much more likely to own labeled neighbors comparing with lower degree ones.
In training process, GCNs use back-propagation to update its neural parameters such that the classification error on labeled nodes is reduced. Thanks to the message-passing mechanism, nodes with labeled neighbors participate more frequently in the optimization process. As a result, GCNs performs better on high-degree nodes.


\subsection{Bridging Node Degree and  Biases in GCNs}
Inspired by \citeauthor{koh2017understanding} \cite{koh2017understanding} and \citeauthor{xu2018representation} \cite{xu2018representation}, we borrow  ideas of sensitivity analysis and influence functions in statistics field to measure the influence of a specific node to the accuracy of GCNs.
We first define node influence from node $v_i$ to $v_k$ as follows:
\begin{equation}
    I(i,k) = \Vert \mathop{\mathbb{E}}(\partial \x_i^{L} / \partial \x_k) \Vert,
\end{equation}
which measures how the feature of $v_i$ affects the training of GCN on node $v_k$. 
Because the loss function is defined purely on labeled nodes, the influence of any unlabeled node (say $v_i$) to the whole GCN can be approximated by the overall influence of every labeled node:
\begin{equation}
    S(i) = \sum_{v_k\in \V^L} I(i,k).
\end{equation}
We can summarize the relation of node degree and the performance of GCNs in the following theorem:
\begin{theorem}
Assume ReLU is the activation function. Let $v_i$ and $v_j$ denote two nodes in a graph. If we have $d_i > d_j$, then the influence score follows: $S(i) > S(j)$ of an untrained GCN.
\end{theorem}
\begin{proof}

The partial differential between $\x_i^{l}$ and $\x_k$ is derived as:
\begin{align}
    \frac{\partial \x_i^{l}}{\partial \x_k} = \frac{1}{\sqrt{d_i}} \cdot \text{diag}(\mathds{1}_{\sigma_l}) \cdot \W^k \cdot \sum_{v_n \in \N(i)}   \frac{1}{\sqrt{d_n}}\frac{\partial \x_n^{l-1}}{\partial \x_k},
\end{align}
where $\sigma_l$ denote the output from the activation function (i.e. ReLU) at the $l$-th GCN layer, and $\text{diag}(\mathds{1}_{\sigma_l})$ is a diagonal mask matrix representing the activation result. Using chain rule, we further derive:
\begin{align}
    \frac{\partial \x_i^{L}}{\partial \x_k} = \sqrt{d_i d_k} \cdot \sum_{p=1}^{\Psi} \prod_{l=L}^0 \frac{1}{d_{p^l}} \text{diag}(\mathds{1}_{\sigma_l}) \cdot \W^l,
\end{align}
where $\Psi$ is the set of all $(L+1)$-length random-walk paths on the graph from node $v_i$ to $v_k$, and $p^l$ represents the $l$-th node on a specific path $p$  ($p^L$ and $p^0$ denote node $i$ and $k$, accordingly).
Note that every path is fully-connected where $v_{p^l} \in \N(p^{l+1})$ for any $p$ and any $l$.
Similar to \citeauthor{xu2018representation} \cite{xu2018representation}, the expectation of ${\partial \x_i^{L}}/{\partial \x_k}$ can be estimated as follows:
\begin{align}
     \nonumber \mathop{\mathbb{E}}\left(\frac{\partial \x_i^{L}}{\partial \x_k}\right) & = \sqrt{d_i d_k} \cdot \sum_{p=1}^{\Psi} \mathop{\mathbb{E}}\left(\prod_{l=L}^0 \frac{1}{d_{p^l}} \text{diag}(\mathds{1}_{\sigma_l}) \cdot \W^l\right) \\
     & = \rho \sum_{v_n \in \N(i)} \sum_{p=1}^{\Psi_n} \mathop{\mathbb{E}}\left(\prod_{l=L-1}^0 \frac{1}{d_{p^l}} \text{diag}(\mathds{1}_{\sigma_l}) \cdot \W^l\right),
\end{align}
where $\rho =({\sqrt{d_k}}/{\sqrt{d_i}}) \cdot \text{diag}(\mathds{1}_{\sigma_L}) \cdot \W^L$ only correlated to $v_i$ and $v_k$, and $\Psi_n$ denote the set of all $L$-length walks from a neighborhood of $v_i$ to $v_k$. Assume the neighborhoods are randomly distributed (i.e., $v_n$ is (near) randomly sampled), the expectation on walks starting from neighborhoods can be replaced by a constant value $\nu$:
\begin{align}
     \sum_{p=1}^{\Psi_n} \mathop{\mathbb{E}}\left(\prod_{l=L-1}^0 \frac{1}{d_{p^l}} \text{diag}(\mathds{1}_{\sigma_l}) \cdot \W^l\right)  = \nu,
\end{align}
and we further have:
\begin{align}
    \nonumber \mathop{\mathbb{E}}\left(\frac{\partial \x_i^{L}}{\partial \x_k}\right) = \rho d_i \nu = \nu\sqrt{d_k d_i} \cdot \text{diag}(\mathds{1}_{\sigma_L}) \cdot \W^L \propto \sqrt{d_i},
\end{align}
therefore, if $d_i > d_j$, then we have $\mathop{\mathbb{E}}\left(\frac{\partial \x_i^{L}}{\partial \x_k}\right) > \mathop{\mathbb{E}}\left(\frac{\partial \x_j^{L}}{\partial \x_k}\right)$. By summing up over all labeled nodes in $\V^L$, we have $S(i) > S(j)$.
\end{proof}
We validate our conclusion in Figure \ref{fig:if_score}. 
\begin{figure}
    \centering
    \begin{subfigure}{0.9\columnwidth}
        \centering
        \includegraphics[width=.5\columnwidth]{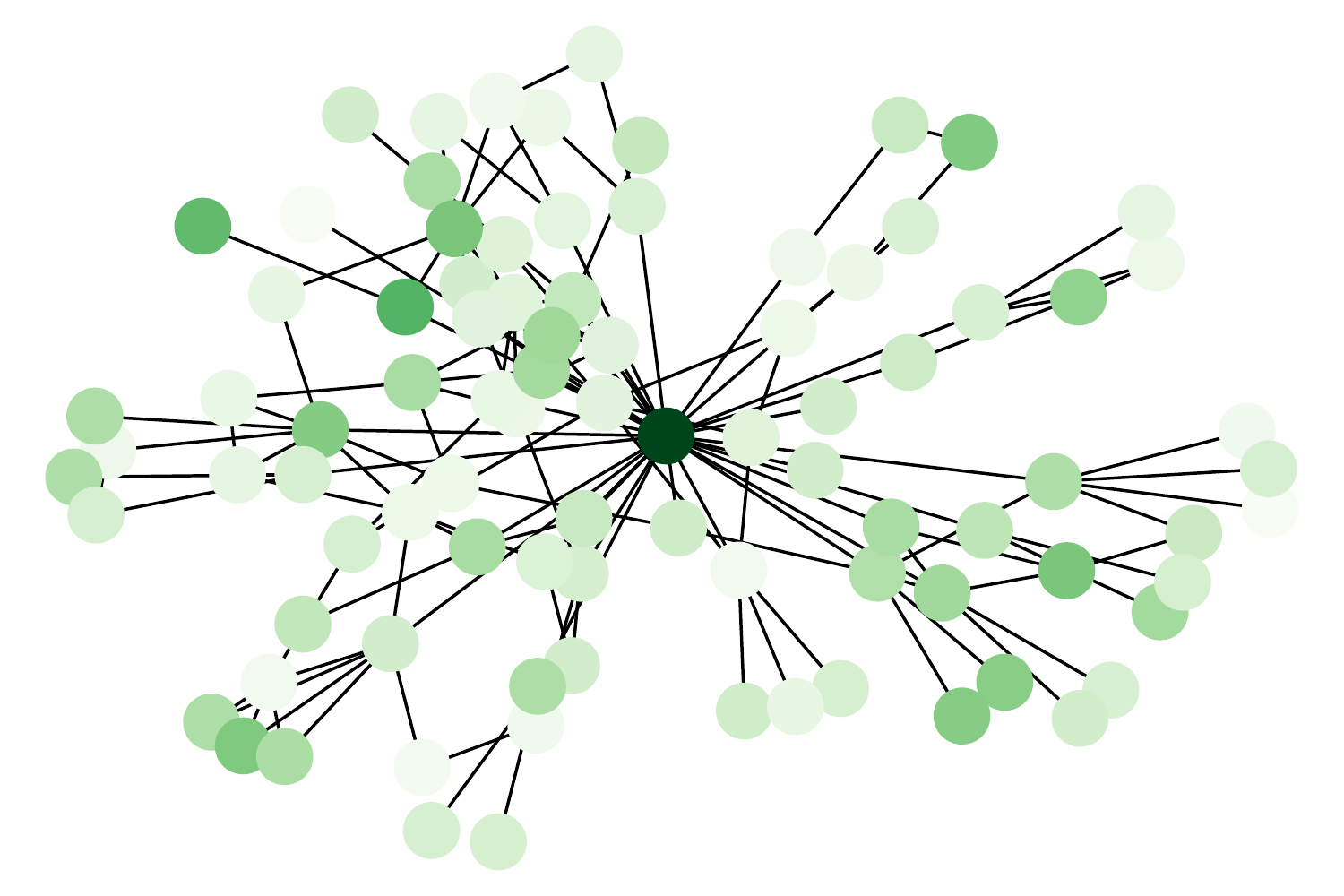}
        \caption{Topology of influence score on a subgraph of Cora. Darker colors denote higher influences.}
        \label{fig:if_score_example}
    \end{subfigure}
    \vskip 0.5em
    \begin{subfigure}{0.9\columnwidth}
        \centering
        \includegraphics[width=.5\columnwidth]{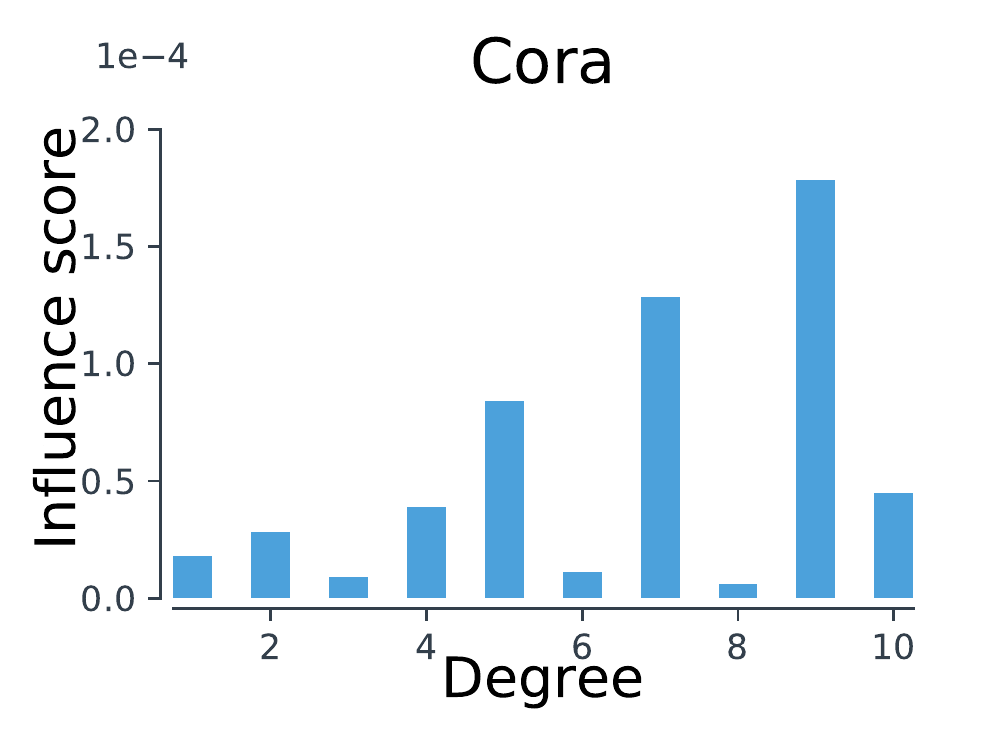} ~~
        \includegraphics[width=.5\columnwidth]{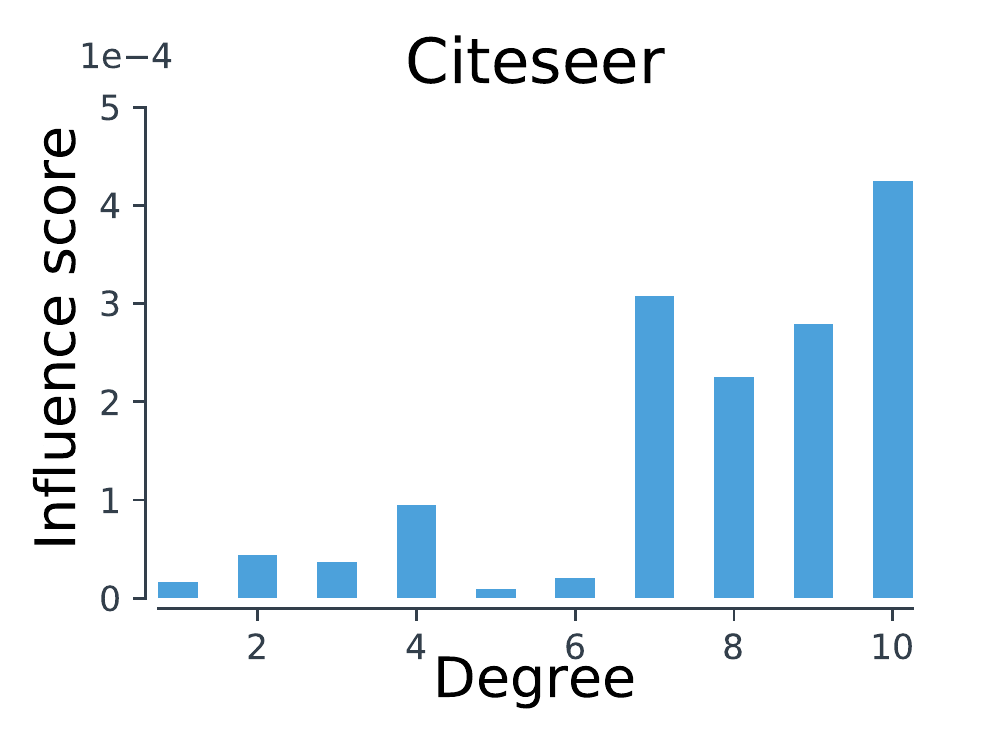}
        \hfill
        \includegraphics[width=.5\columnwidth]{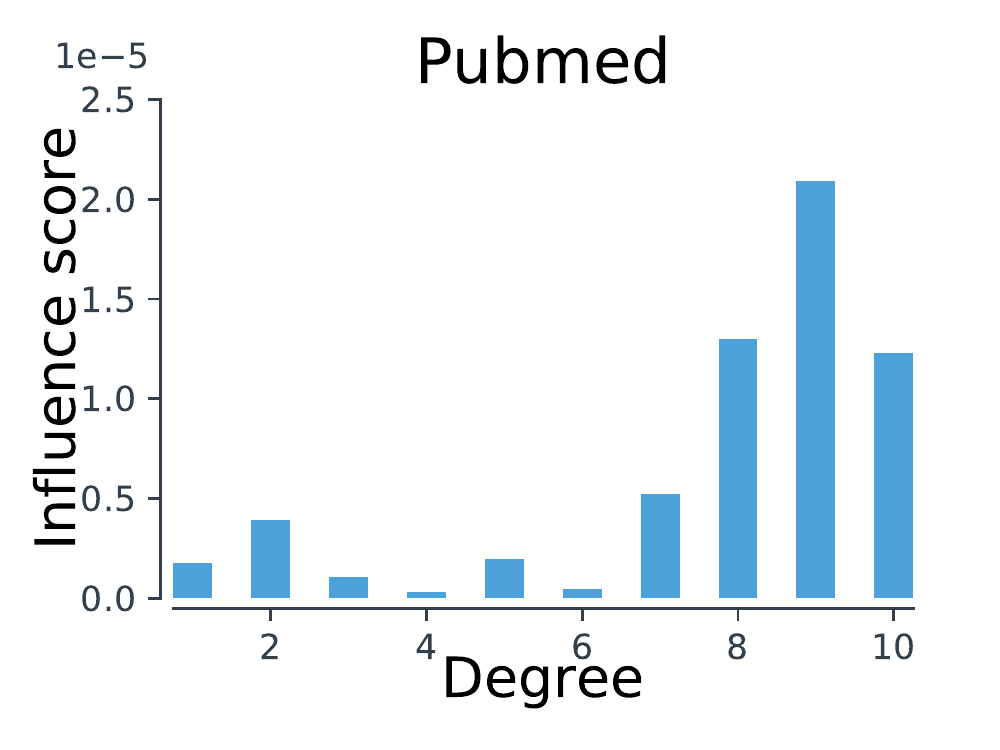} ~~
        \includegraphics[width=.5\columnwidth]{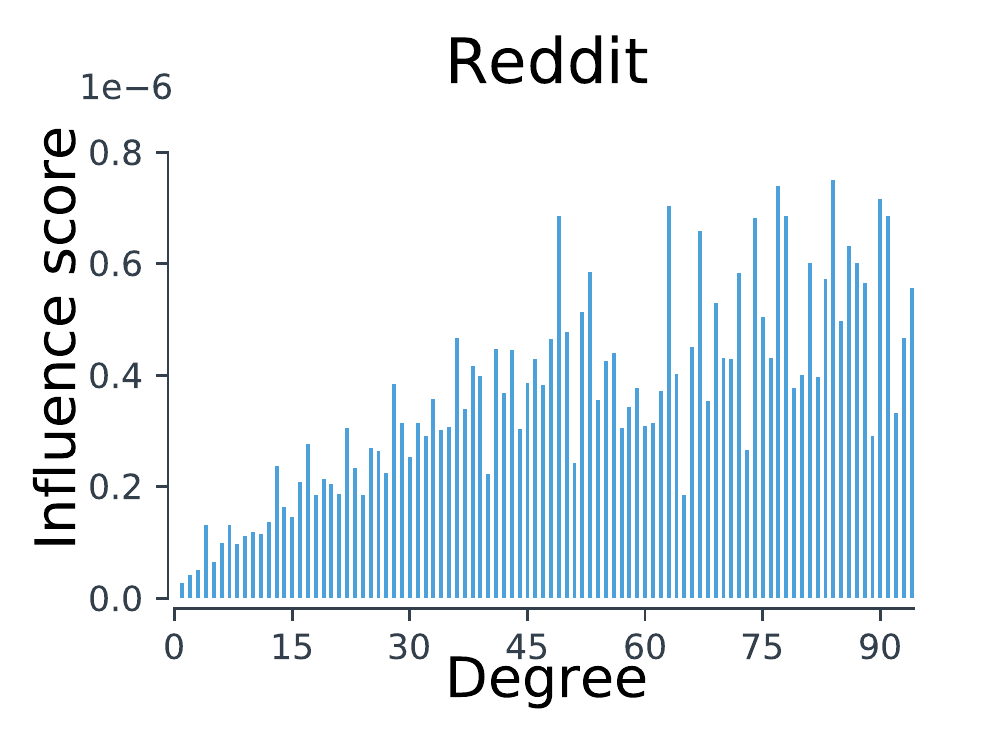}
        \caption{Distribution of influence score varying node degree.}
        \label{fig:if_score_dist}
    \end{subfigure}
    \vskip -1em
    \caption{Distributions of the Influence Score.}
    \label{fig:if_score}
    \vskip -2em
\end{figure}

We first visualize the influence score distribution on a subgraph of the Cora dataset in Figure \ref{fig:if_score_example}. Clearly, the hub node at the center of the graph gains a much higher influence score than others.
We further analyze the distribution of the influence score on four datasets, and report the results in Figure \ref{fig:if_score_dist}.
Clearly, the influence score increases as the node degree becomes larger. This indicates that nodes with larger degrees have higher impact on the training process of GCN, resulting in imbalanced error rate distribution over different degrees.

\section{Approach}
With the above analysis, we summarize the limitations of GCNs as follows: (1)  GCNs use the same set of parameters for all nodes and fails to model the local intra- and inter- relations of nodes, resulting in model-aspect biases; (2) low degree nodes are less likely to have labeled neighbors and participate inactively when training GCNs, such biases come from the data distribution aspect. 
To address these issues, we propose \ours that improves GCNs from two folds:
\textit{Firstly}, we propose a degree-specific GCN (\dsgnn) layer whose parameters are generated by a recurrent neural network (RNN). Nodes with different degrees have their own specific parameters so that the local intra-relation is captured. Besides, as parameters are iteratively generated from the same RNN, their inner correlations help model the inter-relation of nodes with similar degrees. The \dsgnn layer balances the global generalization and local discrepancies for nodes with various degrees.
\textit{Secondly}, we design a self-supervised-learning algorithm to construct pseudo labels with uncertainty within unlabeled nodes. This is achieved by training a Bayesian neural network (BNN). The \dsgnn is fine-tuned on both true and pseudo labels, where the artificial ones are weighted according to their uncertainties. This prevents \ours from overfiiting to inaccurate pseudo labels.

\subsection{Degree-Specific GCN Layer}
As the training of a GCN is dominated by high-degree nodes, using one set of parameters could lead to sub-optimal results. To increase the diversity of learned parameters for nodes with different degrees, following aggregation can be used to distinguish the degree-specific information from the graph:
\begin{equation}
\label{eqn:sliced_layer}
    \x_i^{l+1} = \sigma \Big(\sum_{j \in \N(i)} a_{ij}(\W^l + \W_{d(j)}^l) \x_j^l\Big),
\end{equation}
where $\W_{d(j)}^l$ captures degree-specific information. $\W^l$ is the original GNN parameters at layer $l$ in Eqn \ref{eqn:general_layer}.

The design of $\W_{d(j)}^l$ is a non-trivial task. One straight-forward way is making degree-specific parameters unique for all degrees. However, the maximum value of node degree on a graph can be extremely large due to the long-tail power-law distribution, constructing unique parameters for every degree is impractical. 
Besides, some higher degrees are underrepresented, with only few nodes available. How to prevent underfitting issue for them is also a challenging problem.
To overcome this issue, \citeauthor{wu2019demo} \cite{wu2019demo} propose a hashing-based solution where some degrees are mapped to the same entry of a hash table containing multiple sets of GCN parameters. By manually tuning the size of the hash table, the total number of degree-specific parameters is under control. 

\begin{figure}[!t]
    \centering
    \includegraphics[width=.8\columnwidth]{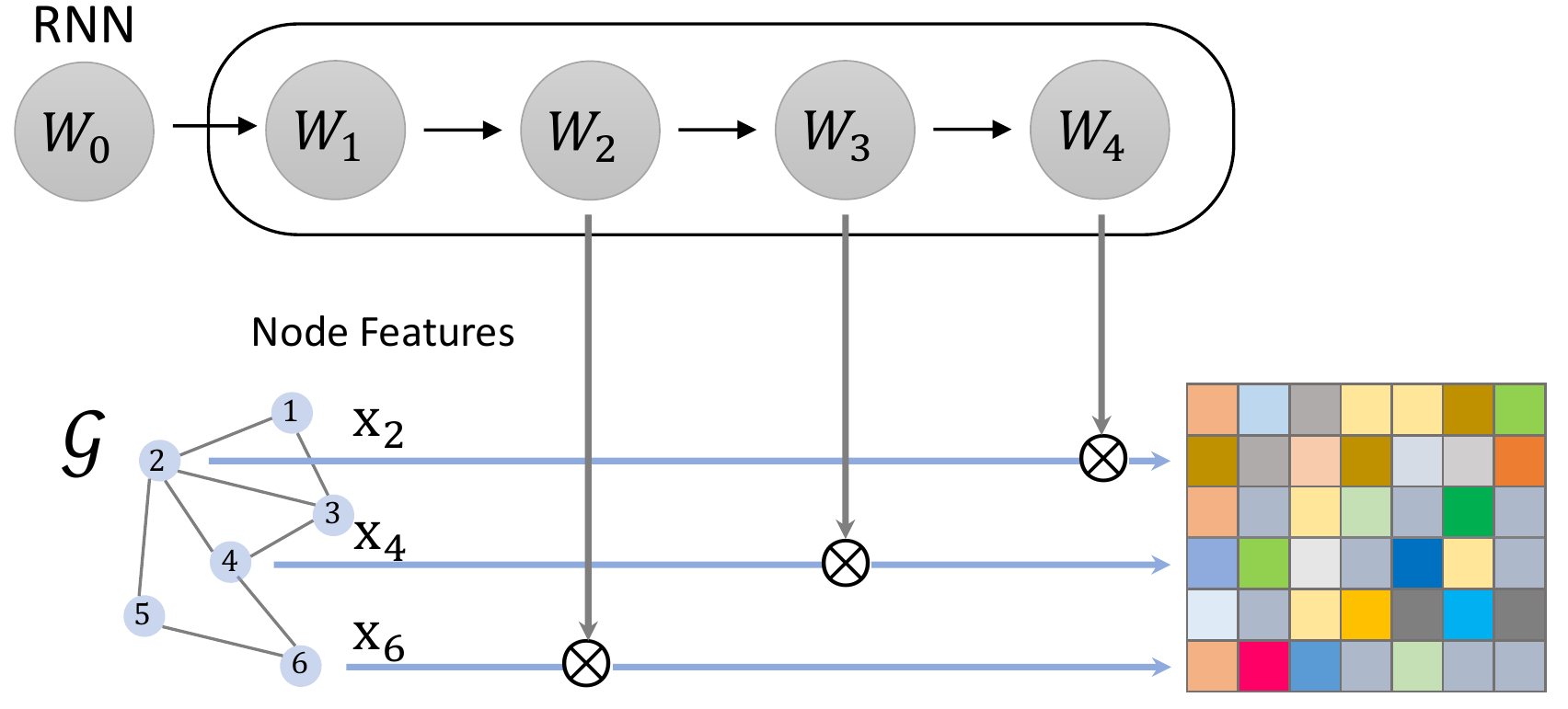}
    \vskip -1em
    \caption{GNN with degree-specific trainable parameters. Node features multiply with different parameters generated by the RNN according to their degree.}
    \label{fig:framework_gnn}
    \vskip -1.2em
\end{figure}

However, the hashing-based approach randomly maps node degree to parameters, and ignores the local inter-relations of nodes with similar degrees. If two nodes have close degree values, their may have a tight correlation. The necessity of capturing local inter-relation of nodes motivates us to adopt an RNN to generate the degree-specific parameters, which is shown in Figure \ref{fig:framework_gnn}.
Specifically, let $\W_0^l$ denote the initialization input to an RNN cell $\text{RNN}(\cdot)$, degree-specific parameters are generated as follows:
\begin{equation}
    W_{k+1}^l = \text{RNN}(W_{k}^l), \ k=0, 1, \cdots, d_{\text{max}},
\end{equation} 
where $W_{k+1}^l$ is the updated hidden state of the RNN after feeding $W_{k}^l$ as the input,
and $d_{\text{max}}$ is a threshold to prevent long-tail issue of the degrees. Nodes with degree higher than $d_{\text{max}}$ are processed using $W_{\text{max} + 1}^l$.
The generated parameters can cover every degree. The advantages of using an RNN are (1) as RNN is iterating over all degrees, generated degree-specified parameters are correlated with each other corresponding to the degree so that modeling local inter-relations of nodes is guaranteed; (2) the total number of actual trainable parameters is fixed (i.e., the initialization input and parameters in the RNN cell), which is more efficient comparing with setting up every $\W_{d(i)}^l$ separately or use a hashing table.
Note that the generated parameters from RNN naturally capture the local intra-relation because every degree has its unique parameters. Besides, the shared parameters $\W^l$ handles the globally shared node relations.

While the \dsgnn layer reduces degree-related biases in GCNs from the model aspect, 
low-degree nodes still participate less frequently when training the \dsgnn. To provide sufficient supervisions for low-degree nodes, we introduce a self-supervised-learning algorithm that creates high-quality pseudo-labels on unlabeled nodes.

\begin{figure*}[t]
    \centering
    \begin{subfigure}[b]{\columnwidth}
        \centering
        \includegraphics[width=.85\columnwidth]{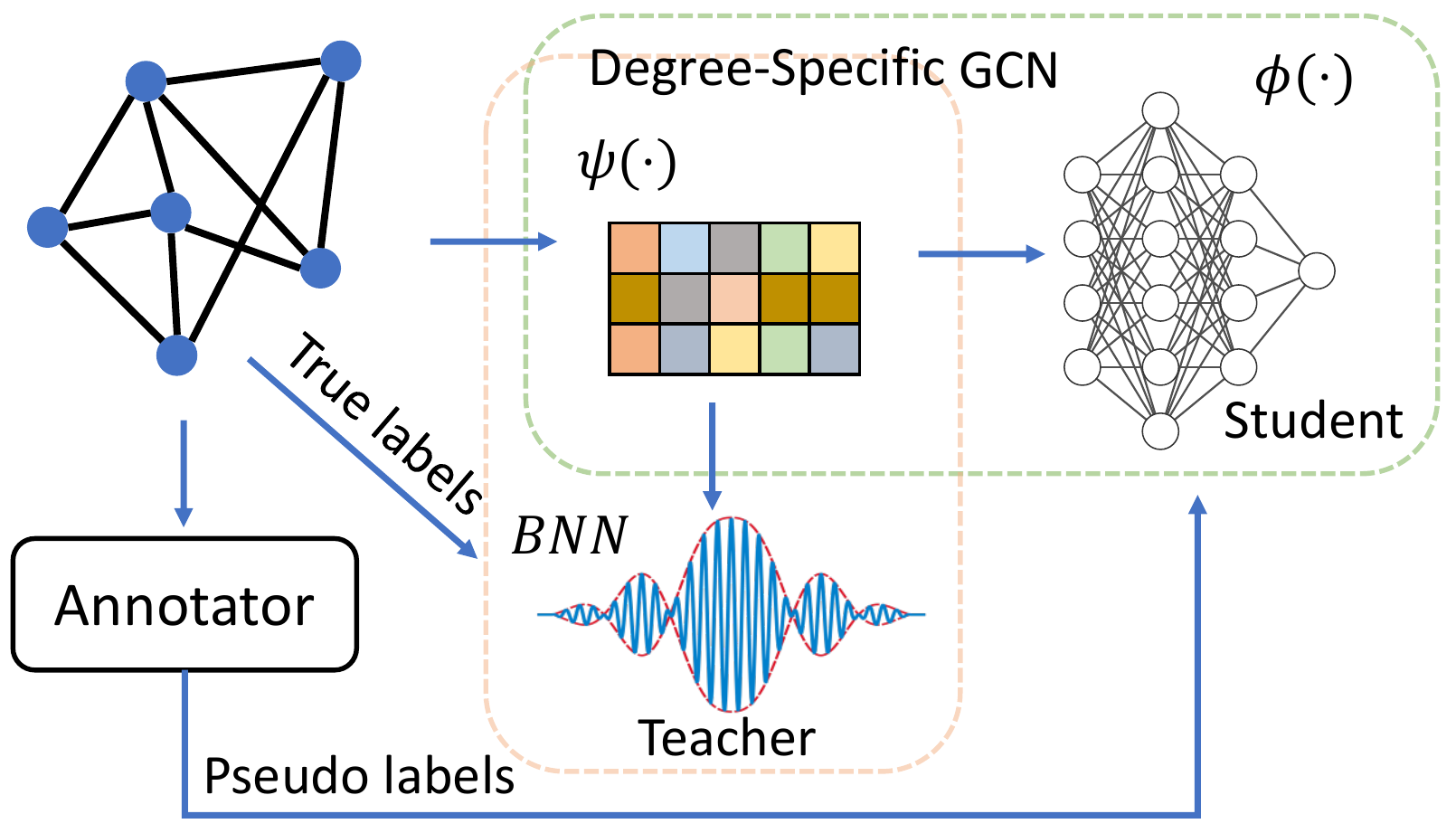}
        \caption{Pre-train student and teacher.}
        \label{fig:framework_pretrain}
    \end{subfigure}
    \begin{subfigure}[b]{\columnwidth}
        \centering
        \includegraphics[width=0.85\columnwidth]{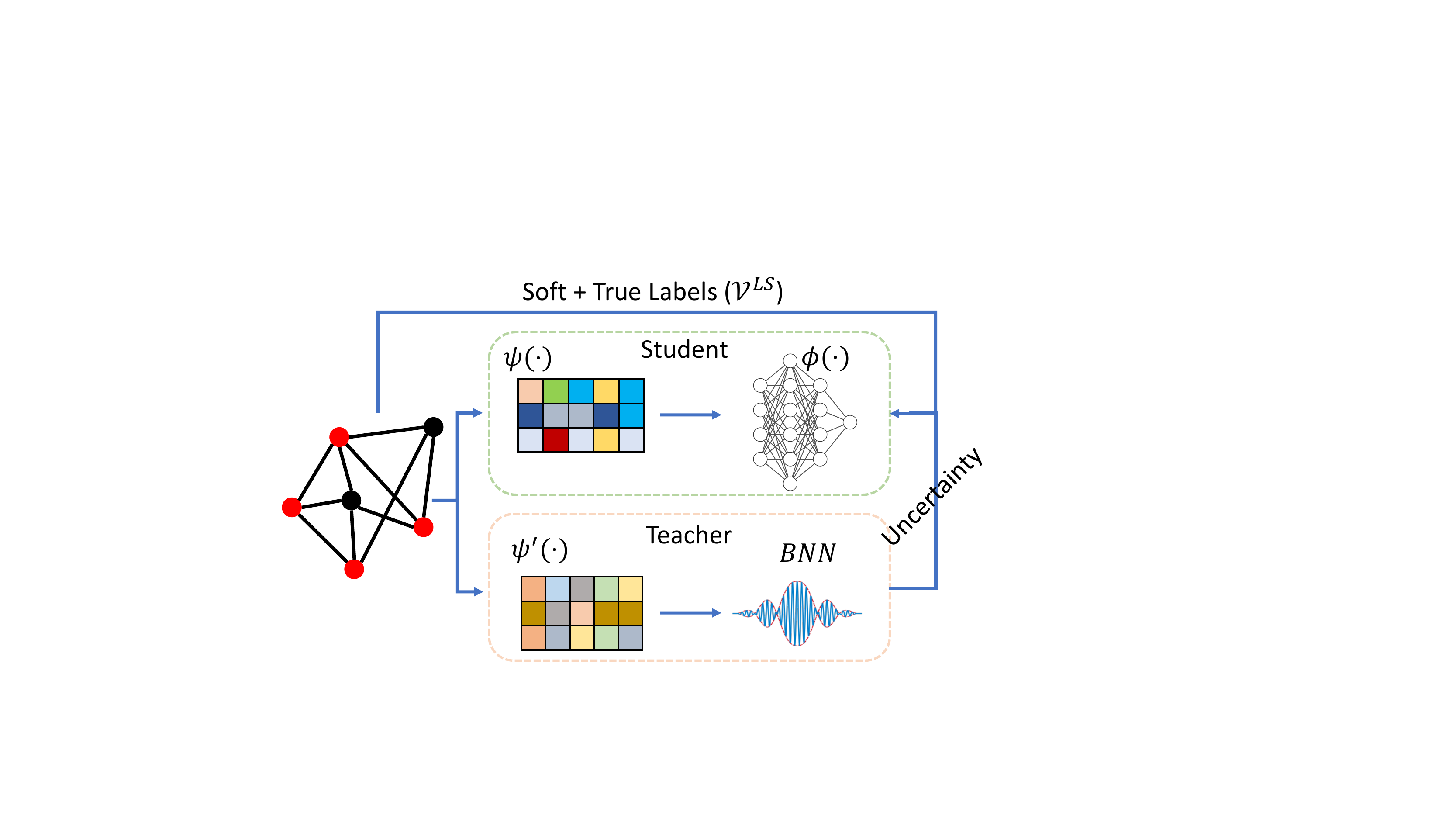}
        \caption{Finetune the student on $\V^{LS}$ with dynamic step size.}
        \label{fig:framework_finetune}
    \end{subfigure}
    \vskip -1em
    \caption{Overall framework of \ours.}
    \label{fig:framework}
    \vskip -1em
\end{figure*}

\subsection{Self-Supervised-Training with Bayesian Teacher Network}
In most semi-supervised settings on graph data, the number of unlabeled nodes is much larger than that of labeled ones (i.e., $|\V^L| \ll |\V^U|$. 
We assume the existence of a graph annotator that can heuristically generate pseudo-labels for nodes in $\V^U$, such as propagation algorithm \cite{zhu2002learning}, label spreading \cite{zhou2004learning}, and PairWalks \cite{wu2012learning}. The pseudo-labels are noisy and less accurate compared with the true labels from $\V^L$ because of the limitations of the annotator. 
The intuition of proposed self-learning algorithm is to leverage the large amount of pseudo-labels in the training of GCNs so that even for low-degree nodes can have labeled neighbors. However, different from existing literature \cite{sun2019multi,li2018deeper} that use pseudo labels in the same way of true labeled nodes, we also judge the quality of pseudo labels to avoid overfitting on inaccurate pseudo labels.
Specifically, we design a Bayesian neural network as a \textit{teacher} to justify the quality of pseudo-labels from the annotator,
so that the GCNs as a \textit{student} can fully exploit the pseudo-labels. There are two steps of the self-learning process as illustrated in Figure \ref{fig:framework}.

\subsubsection{Pre-training with the Annotator}
Firstly, we build the student network using the proposed degree-specific GNN layer. 
As shown in Figure \ref{fig:framework_pretrain}, the student first applies multiple \dsgnn layers over the input graph ($\psi(\cdot)$ part) to capture the dependencies of graph structure and to model the correlation among nodes with different degrees.
Taking the graph $\G$ as an input, $\psi(\cdot)$ transform each node into its representation vector.
To further classify each node, we then apply fully-connected layers followed by a softmax layer ($\phi(\cdot)$ part) on representation vectors from $\psi(\cdot)$. Different from conventional GNNs, the student network leverage $\psi(\cdot)$ to learn data representation from the graph, and assign the classification task to the second part $\phi(\cdot)$. Using the pseudo labels from the annotator, we pre-train the student network so that $\psi(\cdot)$ is fitted to the data and $\phi(\cdot)$ becomes a noisy classifier.
The whole student network is represented by $\phi(\psi(\cdot))$.

However, simply treating all pseudo labels as ground truth will hurt the performance. We then design a teacher network to estimate the uncertainty of pseudo labels from the annotator.
The teacher network is constructed based on a Bayesian neural network (BNN) \cite{dehghani2017fidelity}. We use the node representation from the data representation learner $\psi(\cdot)$ as the input, to train a fully-connected BNN using real-world truely labeled nodes $\V^L$, as illustrated in Figure \ref{fig:framework_pretrain}. 
In particular, the BNN aims at learning the posterior distribution of its parameters, defined as follows:
\begin{equation}
    p(\zeta | \psi(x)) \propto p(\psi(x)|\zeta)\cdot p(\zeta),
\end{equation}
where $\zeta$ denotes the parameters of the BNN, $p(\zeta)$ is the prior of $\zeta$ that contains our assumption of the network parameters, and $p(\psi(x)|\zeta)$ is the likelihood which describe the input data (i.e., node representation from $\psi(x)$). 
The probability distributions of model parameters $\zeta$ are updated with the Bayes theorem taking into account both the prior and the likelihood.
Without loss of generalities, we use normal distribution as the prior for the BNN. We fix the representation learner when updating the BNN part, so that the teacher can leverage the knowledge from the annotated results. Besides, training on top of $\psi(\cdot)$ ensures the teacher is learning in the same representation space of the student, so that the judgement of unlabeled nodes in further steps is unbiased and has no domain shifting for the student network. 
We use a two-layer fully-connected network as the approximation for the likelihood.
The posterior mean $\mu$ and posterior covariance $\kappa$ of the BNN is acquired after training the BNN model, and are further used to create soft labels on unlabeled nodes with uncertainties. In particular, for every unlabeled node $v_i \in \V^U$, we acquire its prediction and uncertainty score as follows:
\begin{align}
    \nonumber y_i^s = f(\mu(x_i)), \quad  c_i = g(\kappa(x_i)),
\end{align}
where $f(\cdot)$ and $g(\cdot)$ are two functions (e.g., neural networks) that map the posterior mean and covariance vectors to desired soft label and uncertainty score.

We visualize the prediction and uncertainty of the teacher BNN trained on a small subset from the reddit network dataset in Figure \ref{fig:teacher_analysis}. As we can see in Figure \ref{fig:uncertainty}, the uncertainty for labeled nodes are almost zero, indicating the teacher fit the training data very well. Meanwhile, we also observe that the uncertainty scores on low degree nodes tend to be larger, which is consistent with our previous analysis. As low degree nodes have less impact on the training loss function and receive less supervision from labeled neighbors, it is harder to generate a confident prediction for them. Similarly in Figure \ref{fig:teacher_classification}, it is more likely for low degree nodes to be misclassified than high degree ones.

\begin{figure}[t]
    \centering
    \begin{subfigure}{.49\columnwidth}
        \centering
        \includegraphics[width=\columnwidth]{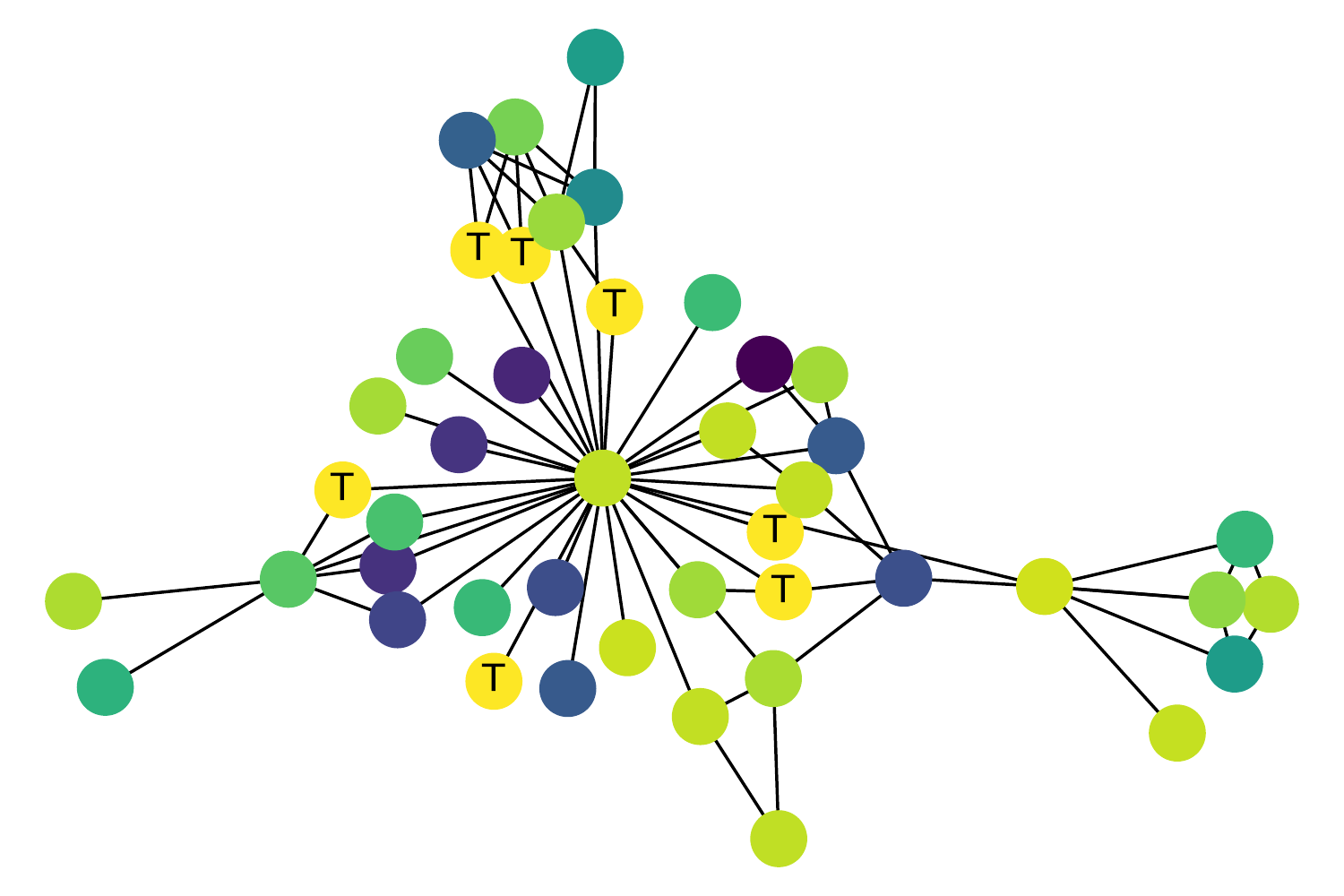}
        \caption{Uncertainty scores from the teacher network. Darker color means higher uncertainty and ``$\top$'' denotes training nodes.}
        \label{fig:uncertainty}
    \end{subfigure}
    \begin{subfigure}{.49\columnwidth}
        \centering
        \includegraphics[width=\columnwidth]{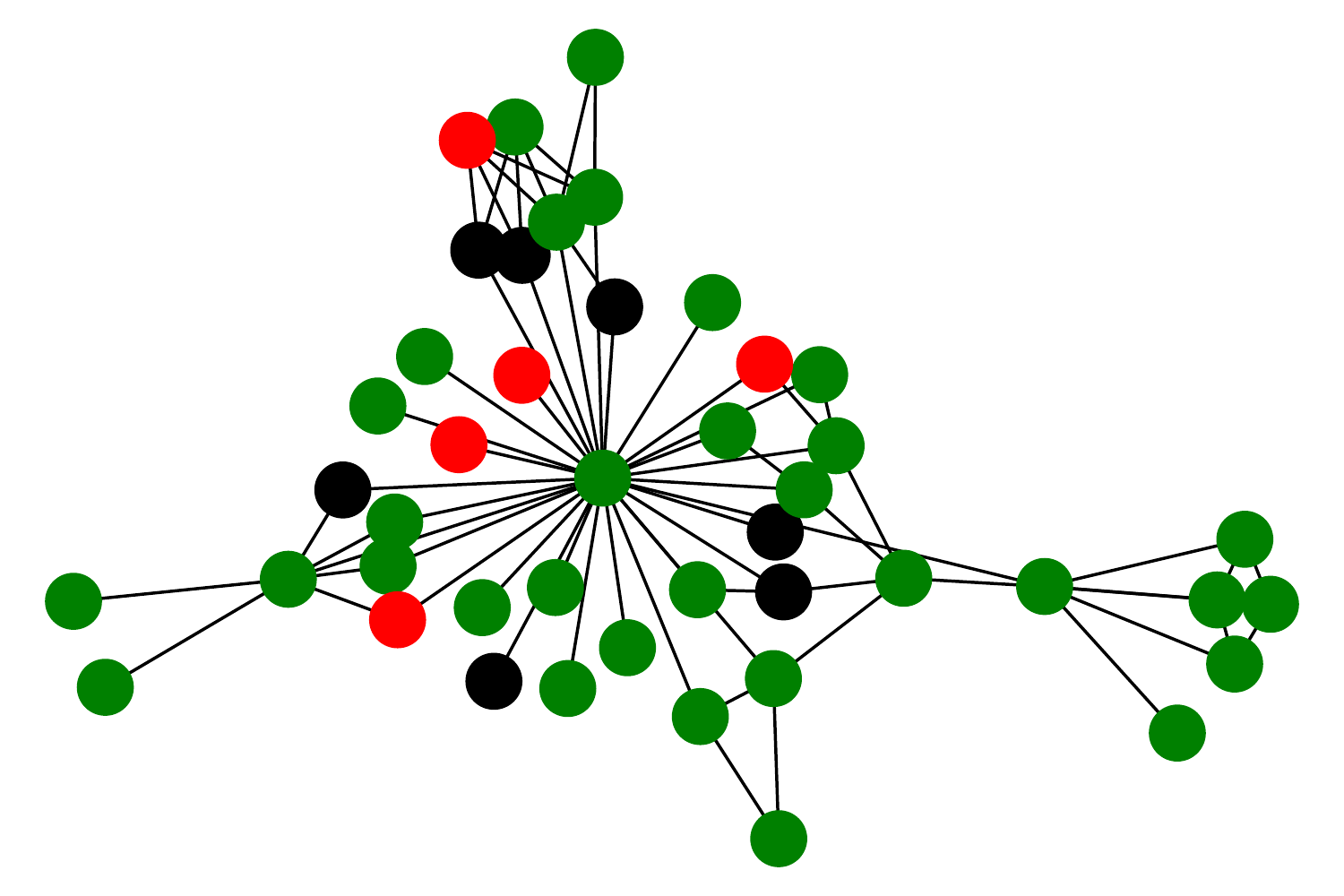}
        \caption{Classification error of the teacher network. Red and green denote wrong and correct prediction respectively, and black represents training nodes. }
        \label{fig:teacher_classification}
    \end{subfigure}
    \vskip -1em
    \caption{Uncertainty score and error distribution of the teacher network. Generally, nodes closer to labeled (training) ones tend to have lower uncertainty and error rate.}
    \label{fig:teacher_analysis}
    \vspace{-1em}
\end{figure}

\subsubsection{Fine-tuning Student with Uncertainty Scores.}
After the pre-training of student and teacher network, the second step of the self-learning process is fine-tuning the student network using generated labels and uncertainty scores from the teacher.
We define a softly-labeled node set $\V^S \subset \V^U$ where nodes in $\V^S$ are labeled identically by both the annotator and the teacher. The intuition is similar to majority vote. Given large amount of unlabeled nodes, it is worthwhile to compile a cleaner labeled node set as a compensation to the existing true labeled nodes.

Existing works exploring self-learning for GNNs treat selected pseudo labels in the same way of using labeled nodes. For example, \citeauthor{li2018deeper} \cite{li2018deeper} and \citeauthor{sun2019multi} \cite{sun2019multi} progressively add selected nodes with pseudo labels into the training set. 
However, such solutions are sub-optimal. One bottleneck is that all selected pseudo labels are equally treated, and are utilized in the same way of true labeled nodes. However, even for pseudo labels with high confidence, they still contain more noise than the real labeled part.

Fortunately, the proposed BNN-based teacher network naturally solves the above challenge. The generated uncertainty scores can be utilized when training with pseudo labels. Specifically, we fine-tune the student network on $\V^{LS} = \V^L \cup \V^S$ using stochastic gradient descent (SGD) algorithm, where the uncertainty score controls the step size for each nodes in $\V^{LS}$. 
We use $\theta$ to denote parameters in the student network, the optimization (learning) goal is as follows:
\begin{equation}
    \theta^* = {\text{argmax}}_{\theta} \L(\theta) = \sum_{v_i \in \V^{LS}}\L(v_i; \theta).
\end{equation}
The updating rule for parameters $\theta$ is:
\begin{equation} \label{eqn:update_theta}
    \theta^\prime = \theta - \sum_{v_i \in \V^{LS}} \eta_i  \L(v_i; \theta),
\end{equation}
where $\eta_i$ is a dynamic step size defined as follows:
\begin{equation} \label{eqn:dynamic_step}
    \eta_i = \eta \cdot \eta_i^c \cdot \eta_i^d = \eta \cdot \text{exp}(-\alpha c_i) \cdot \text{exp}(\beta d_i),
\end{equation}
which contains three parts. The first part $\eta$ is the global step size used in classic SGD. The second part $\eta_i^c$ penalize each sample (node) by its quality, using the uncertainty score acquired from the teacher network. We choose a negative exponential function over the uncertainty score so that nodes with larger uncertainty participate less in the updating process. The third term empirically assigns larger weights to nodes with higher degrees according to the observations in Figure \ref{fig:if_score_example} and Figure \ref{fig:teacher_analysis}. 
Here $\alpha$ and $\beta$ are hyperparameters that balance three parts in the dynamic step size. Generally, larger values of $\alpha$ and/or $\beta$ pay more attention to the uncertainty scores and the degree distribution, correspondingly.
After fine-tining on $\V^{LS}$ using SGD with dynamic step size, we use the student network to predict node labels.

\subsection{Training Algorithm}
We summarize the self-learning process in Algorithm \ref{alg}.
Line 1-3 are the pre-training of student and teacher network. After acquiring predictions and uncertainty scores from the pre-trained teacher in Line 4, we compile $\V^{LS}$ using true labels and the softly-labeled nodes (Line 5-6). Finally, as introduced in Line 7-9, the student network is fine-tuned on $\V^{LS}$ with dynamic step size.
Note that although we select GCN as the basis of \ours, the idea of capturing globally shared, local intra- and inter- relations of nodes with an RNN-based parameter generator, and using self-supervised-learning with dynamic step size are model agnostic. Namely, they can also be applied on other GNN models, such as graph attention networks \cite{velivckovic2017graph}, GraphSAGE \cite{hamilton2017inductive}, etc. We leave this part for future work.

\begin{algorithm}
	\caption{Self-learning for \ours}
    \label{alg}
    \KwIn{$\G = (\V, \E, \X)$ }
    \KwOut{Parameters $\theta$ of student network $\phi(\psi(\cdot))$}
	\tcp{Pre-training}
	Acquire pseudo-labels for $\V^U$ using a graph annotator; \\
	Pre-train $\phi(\psi(\cdot))$ on pseudo labels; \\
	Fix $\psi(\cdot)$ and pre-train BNN part of the teacher network; \\
	Acquire prediction $y_i^s$ and uncertainty score $c_i$ for every node in $\V^U$ from the teacher; \\
	\tcp{Fine-tuning}
	Compile a soft-labeled node set $\V^S  \subset \V^U$ where the teacher network agrees with the annotator; \\
	Build $\V^{LS} = \V^L \cup \V^S$ to fine-tune the student network; \\
	\While{not converge}{
	    Compute dynamic step size $\eta_i$ for $v_i \in \V^{LS}$ as $\eta_i = \eta \cdot \eta_i^c \cdot \eta_i^d$; \\
	    Update parameters of the student network as $\theta^\prime = \theta - \sum_{v_i \in \V^{LS}} \eta_i  \L(v_i; \theta)$; \\
	}
\end{algorithm}
\section{Experiments}
In this section, we conduct experiments on real-world datasets to evaluate the effectiveness of \ours. In particular, we aim to answer the following questions:
\begin{itemize}[leftmargin=*]
    \item Can \ours outperform existing self-training algorithms for GNNs on various benchmark datasets?
    \item How do the degree-specific design (\dsgnn), the machine teaching approach, and the dynamic step size contribute to \ours?
    \item How sensitive of \ours is on the selection of softly-labeled node set?
\end{itemize}
Next, we start by introducing the experimental settings followed by experiments on node classification to answer these questions.

\subsection{Experimental Setup}
\subsubsection{Datasets}
For a fair comparison, we adopt same benchmark datasets used by \citeauthor{sun2019multi} \cite{sun2019multi} and \citeauthor{li2018deeper} \cite{li2018deeper}, including Cora, Citeseer, Pubmed \cite{sen2008collective}. Each dataset contains a citation graph, where nodes represent articles/papers and edges denote citation correlation. Node features are constructed using bag-of words features. 
The detailed statistics of the datasets are summarized in Table \ref{tab:dataset}.

\begin{table}[]
    \centering
    \caption{Statistics of the Datasets}
    \label{tab:dataset}
    \vskip -1em
    \begin{tabular}{c|cccc} \toprule
Dataset & Nodes & Edges & Classes & Features \\ \midrule
Cora & 2708 & 5429 & 7 & 1433 \\
CiteSeer & 3327 & 4732 & 6 & 3703 \\
PubMed & 19717 & 44338 & 3 & 500 \\ \bottomrule
    \end{tabular}
    \vskip -1em
\end{table}


\begin{table*}[t]
    \centering
    \caption{Node Classification Performance Comparison on Cora, Citseer and PubMed}
    \vskip -1em
\begin{tabular}{c|ccccc|ccccc | ccc} \toprule
Dataset     &   \multicolumn{5}{c|}{Cora} & \multicolumn{5}{c|}{Citeseer} & \multicolumn{3}{c}{PubMed}\\ \hline
Label Rate             & 0.5\%   & 1\%     & 2\%     & 3\%     & 4\% 	& 0.5\%   & 1\%     & 2\%     & 3\%     & 4\%	& 0.03\%  & 0.06\%  & 0.09\%    \\ \midrule
LP           & 29.05 & 38.63 & 53.26 & 70.31 & 73.47 & 32.10 & 40.08 & 42.83 & 45.32 & 49.01 & 39.01 & 48.7  & 56.73 \\
ParWalks     & 37.01 & 41.40 & 50.84 & 58.24 & 63.78 & 19.66 & 23.70 & 29.17 & 35.61 & 42.65 & 35.15 & 40.27 & 51.33 \\
GCN          & 35.89 & 46.00 & 60.00 & 71.15 & 75.68 & 34.50 & 43.94 & 54.42 & 56.22 & 58.71 & 47.97 & 56.68 & 63.26 \\
DEMO-Net & 33.56 & 40.05 & 61.18 & 72.80 & 77.11 & 36.18 & 43.35 & 53.38 & 56.5 & 59.85 & 48.15 & 57.24 & 62.95  \\ 
Self-Train   & 43.83 & 52.45 & 63.36 & 70.62 & 77.37 & 42.60 & 46.79 & 52.92 & 58.37 & 60.42 & 57.67 & 61.84 & 64.73 \\
Co-Train     & 40.99 & 52.08 & 64.27 & 73.04 & 75.86 & 40.98 & 56.51 & 52.40 & 57.86 & 62.83 & 53.15 & 59.63 & 65.50 \\
Union        & 45.86 & 53.59 & 64.86 & 73.28 & 77.41 & 45.82 & 54.38 & 55.98 & 60.41 & 59.84 & 58.77 & 60.61 & 67.57 \\
Interesction & 33.38 & 49.26 & 62.58 & 70.64 & 77.74 & 36.23 & 55.80 & 56.11 & 58.74 & 62.96 & 59.70 & 60.21 & 63.97 \\
M3S          & 50.28 & 58.74 & 68.04 & 75.09 & 78.80 & 48.96 & 53.25 & 58.34 & 61.95 & 63.03 & 59.31 & 65.25 & 70.75 \\ \midrule
\ours        & \textbf{53.58} & \textbf{61.36} & \textbf{70.31} & \textbf{80.15} & \textbf{81.05} & \textbf{54.07} & \textbf{56.68} & \textbf{59.93} & \textbf{62.20} & \textbf{64.45} & \textbf{61.15} & \textbf{65.68} & \textbf{71.78} \\ \bottomrule
\end{tabular}
    \label{tab:node_classification}
\end{table*}

\subsubsection{Baselines}
We compare \ours with representative and state-of-the-art node classification algorithms, which includes:
\begin{itemize}[leftmargin=*]
    \item LP \cite{zhu2002learning}: Label Propagation is a classical self-supervised learning algorithm which where we iteratively assign labels to unlabelled points by propagating labels through the graph. It serves as the weak annotator in our framework.
    \item ParWalks~\cite{wu2012learning}: ParWalks  extends label propagation by using partially absorbing random walk.
    \item GCN \cite{kipf2016semi}:  GCN is a widely used graph neural network. It defines graph convolution via spectral analysis.
    \item DEMO-Net \cite{wu2019demo}: It proposes multi-task graph convolution where each task represents node representation learning for nodes with a specific degree value, thus leading to preserving the degree specific graph structure. DEMO-net also contains other constraints to improve the representation learning, including order-free and  seed-oriented. These constraints are removed for a fair comparison because they do not tackle the degree-related biases of GCNs, and can be applied on all above methods. We choose the weight version of DEMO-net due to better performances.
    \item Co-Training~\cite{li2018deeper}: This method uses the ParWalk to find the most confident vertices – the nearest neighbors to the labeled vertices of each class, and then add them to the label set to train a GCN.
    \item Self-Training, Union and Intersection \cite{li2018deeper}: Self-training picks the most confident soft-labels of GCN and puts it into the labeled node set to improve the performance of GCN.  Union takes the union of the most confident soft-labels by both GCN and ParWalk as self-supervision while Intersection takes the intersection of the two as the self-supervision.
    \item M3S \cite{sun2019multi}: Multi-Stage Self-Supervised Training leverages DeepCluster technique to provide self-supervision and utilizes the cluster information to iterative train GNN.  
\end{itemize}

\subsubsection{Settings and Hyperparameters}
The training and testing set are generated as follows: we randomly sample $x\%$ of nodes for training, $35\%$ nodes for testing, and treat the remained nodes as unlabeled ones for each dataset. Furthermore, to understand how \ours performs under various label sparsity scenarios in real-world, for CORA and Citeseer, we vary $x$ as $\{0.5, 1, 2, 3, 4\}$. Since PubMed is relative larger than Cora and CiteSeer, we vary $x$ as $\{0.03, 0.06, 0.09\}$  for it. Note that we set $x$ as small values because in typical setting of real-world semi-supervised node classification tasks, only a small amount of nodes are labeled for training~\cite{sun2019multi,li2018deeper}.
We adopt the same hyper-parameters for GCN as introduced by \citeauthor{kipf2016semi} \cite{kipf2016semi}, which is a two-layer GCN with 16 hidden units on each layer. For DEMO-Net, Self-Train, Co-train, Union, and Intersection, we adopt their public code and tune hyperparameters for the best performance. We implement M3S following the descriptions in the paper \cite{sun2019multi}. 
For the student network part, both $\phi(\cdot)$ and $\psi(\cdot)$ are implemented by one \dsgnn layer. We set $d_{\text{max}}$ to 10. The Bayesian neural network part of the teacher contains two fully-connected layers, each contains 16 hidden units. We fix $\alpha$ and $\beta$ to 1.
Note that for fair comparison, we set all self-supervised-learning GCNs to two-layers with 16 hidden units, which is aligned with both GCN and \ours.
We report the averaged results over 10 times of running.

\subsection{Node Classification Performance} 

To answer the first research question, we conduct node classification with comparison to existing self-training algorithms for GNNs on the datasets introduced above. The experimental results in terms of accuracy for the three datasets are reported in Table~\ref{tab:node_classification}. From the table, we make the following observations:
\begin{itemize}[leftmargin=*]
    \item Generally, self-supervision based approaches such as M3S, Intersection and Union outperform algorithms without self-supervision such as LP and GCN, which implies that self-supervision could help provide more labeled nodes to training so that the percentage of labeled neighborhood of low-degree increases.
    \item As label rate $x$ increases, the performance improvement of  self-supervision based approaches over non-self-supervision approaches decreases. For example, on Cora dataset, as $x$ increase from $0.5\%$ to $4\%$, the performance improvement of M3S and SL-DGNN over GCN are $\{14.39, 12.74, 8.04, 3.94, 3.12\}$ and $\{17.69, 15.36,$ $ 10.31,$ $9.00, 5.37\}$, respectively. This is because as the amount of labeled data increases, the percentage of labeled neighborhood of low-degree also increases, which makes the introduction of self-supervision less useful.
    \item For all the three datasets and label rate, \ours consistently outperforms all the baselines significantly, which shows the effectiveness of the proposed framework. In particular, both M3S and \ours adopt self-supervision. \ours significantly outperforms M3S because \ours explicitly model degree-specific GNN layer through LSTM, which could benefit the low-degree nodes more.
\end{itemize}

\subsection{Performance on Low Degree Nodes}
\ours is motivated by the observation that the number of labeled nodes for low-degree nodes is very much smaller than that of high-degree nodes, which makes GNN biased towards high-degree nodes. Thus, degree specific GNN layer and self-training with Bayesian teacher networks are leveraged to alleviate the issue. To validate the effectiveness of the proposed framework \ours on low-degree nodes, we further visualize the node classification performance of low-degree nodes on Cora and Citeseer in Figure \ref{fig:low_degree}. Note that for Cora and Citeseer, $96.45\%$ and $97.53\%$ nodes have a degree less than 11.  From the figure, we observe that:
\begin{itemize}[leftmargin=*]
    \item Both \dsgnn and \ours outperform GNN significantly, especially on node with small degrees, which shows the effectiveness of degree specific layer and self-supervision for improving performance of low-degree nodes. In addition, \ours has better performance than \dsgnn, which implies that the degree specific layer and self-supervision improves the performance from two different perspectives. Degree specific layer tries to learn node-specific parameters to reduce the bias towards high-degree nodes while self-supervision  tries to improve the number labeled nodes in each node's neighborhood.
    \item When degree the node degree is very small, say $\{1,2,3,4,5\}$, the improvement of \dsgnn and \ours is very significant. As the degree become larger, the improvement becomes smaller. This is because when degree is very small, most of these nodes have very few labeled nodes in their neighborhood. A small amount of soft-label and the degree-specific parameters could improve the performance a lot. However, when the degree become larger, there are already enough supervision to train a good GNN, which makes the improvement insignificant. However, as the majority nodes in graphs are low degree nodes, \ours can still improve the overall performance significantly.
\end{itemize}

\begin{figure}[t]
    \centering
    \begin{subfigure}{.49\columnwidth}
        \centering
        \includegraphics[width=\columnwidth]{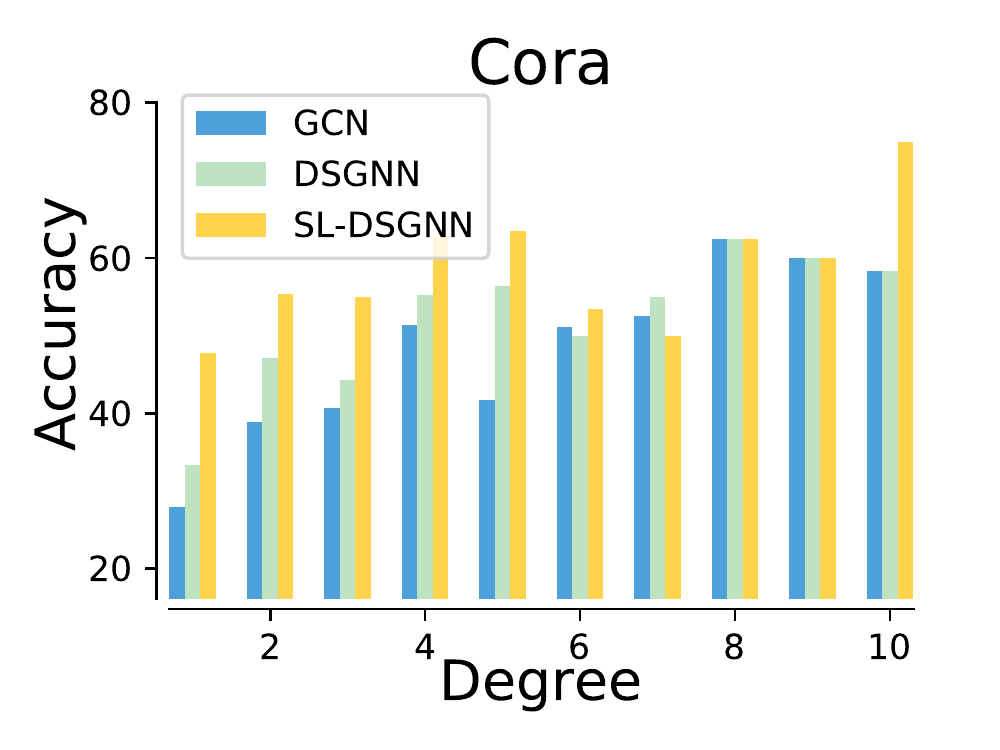}
        \label{fig:low_degree_cora}
    \end{subfigure}
    \begin{subfigure}{.49\columnwidth}
        \centering
        \includegraphics[width=\columnwidth]{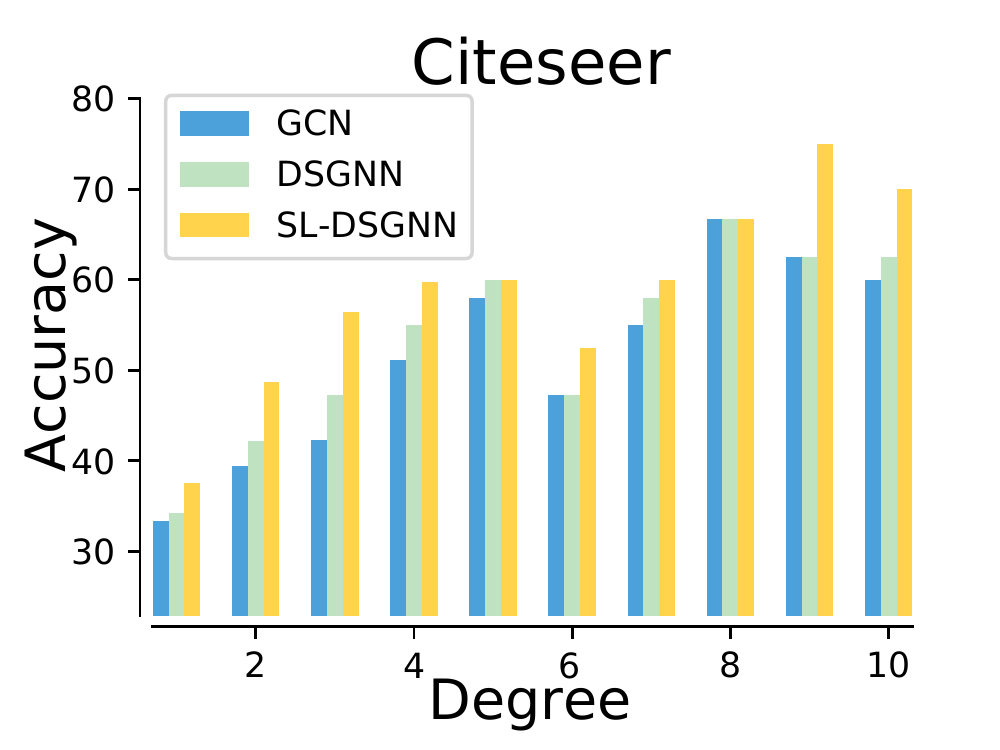}
        \label{fig:low_degree_citeseer}
    \end{subfigure}
    \vskip -2em
    \caption{Node Classification Performance on Nodes with Different Degrees}
    \label{fig:low_degree}
    \vskip -1em
\end{figure}

\begin{table}[t]
    \centering
    \caption{Ablation study on Cora dataset.} \label{tab:abla_cora}
    \vskip -1em
\begin{tabular}{c|ccccc} \toprule
Label Rate             & 0.5\%   & 1\%     & 2\%     & 3\%     & 4\%     \\ \midrule
\dsgnn                 & 36.11 & 47.67 & 61.91 & 73.87 & 77.03  \\
MT-GNN                 & 50.51 & 57.47 & 67.26 & 78.52 & 78.84  \\
$\text{\ours}_{fs}$    & 51.36 & 59.85 & 68.81 & 79.14 & 79.90  \\
SL-GNN                 & 52.05 & 60.41 & 69.51 & 79.75 & 80.21  \\ \midrule
\ours                  & 53.58 & 61.36 & 70.31 & 80.15 & 81.05  \\ \bottomrule
\end{tabular}

\vskip 1em
\centering
    \caption{Ablation study on Citeseer dataset.} \label{tab:abla_citeseer}
    \vskip -1em
\begin{tabular}{c|ccccc} \toprule
Label Rate & 0.5\%   & 1\%     & 2\%     & 3\%     & 4\%     \\ \midrule
\dsgnn    & 37.51 & 44.75 & 55.41 & 56.9 & 60.24 \\
MT-GNN    & 49.78 & 50.75 & 55.14 & 59.01 & 61.23  \\
$\text{\ours}_{fs}$    & 51.89 & 53.26 & 58.38 & 60.63 & 62.15 \\
SL-GNN    & 52.77 & 54.79 & 57.27 & 61.98 & 63.99  \\ \midrule
\ours     & 54.07 & 56.68 & 59.93 & 62.20 & 64.45   \\ \bottomrule
\end{tabular}
\vspace{-1em}
\end{table}

\subsection{Ablation Study}
In this subsection, we conduct ablation study to understand the impact of degree-specific GNN, the dynamic step size for SGD, and the self-teaching algorithm, which answers the second research question. 
Specifically, several variations of \ours are compared including (1): \dsgnn which applies the degree-specific parameters on GCN; (2) MT-GNN which replace the dynamic step size with original one and remove the softly-labeled node set from $\V^{LS}$ (i.e., only use the labeled nodes for fine-tuning the student network). MT-GNN can be treated as a GNN enhanced by the vanilla machine teaching algorithm; (3) $\text{\ours}_{fs}$ which removes the dynamic step size; and (4) SL-GNN which removes the degree-specific design in the student network. The performance of \ours and the variants on Cora and Citeseer are reported in Table~\ref{tab:abla_cora} and \ref{tab:abla_citeseer}, respectively. From these two tables, we observe that: (i) In terms of the comparison between SL-GNN and \ours, \ours performs slightly better than SL-GNN, which shows that degree specific layer can slightly improve the performance; (ii) In terms of the comparison between $\text{\ours}_{fs}$ and \ours, \ours has better performance than $\text{\ours}_{fs}$, which is because $\text{\ours}_{fs}$ doesn't adopt the dynamic step size; and (iii) \ours significantly outperforms \dsgnn, which shows the effectiveness of the proposed self-supervised training.

\begin{table}[]
    \centering
    \caption{Influence of the softly-labeled node set.}
    \vskip -1em
\begin{tabular}{c|c|ccccc} \toprule
Dataset & Node set             & 0.5\%   & 1\%     & 2\%     & 3\%     & 4\%     \\ \midrule
\multirow{4}{*}{Cora} & \dsgnn & 36.11 & 47.67 & 61.91 & 73.87 & 77.03\\ 
& $\V^S_A$     & 47.21 & 55.10 & 67.15 & 76.39 & 75.07 \\
& $\V^S_T$     & 50.73 & 58.29 & 68.85 & 77.24 & 76.93  \\
& \ours           & 53.58 & 61.36 & 70.31 & 80.15 & 81.05 \\\midrule
\multirow{4}{*}{Citeseer} & \dsgnn & 37.51 & 44.75 & 55.41 & 56.9 & 60.24 \\
& $\V^S_A$    & 50.68 & 53.42 & 57.10 & 60.52 & 60.63  \\ 
& $\V^S_T$    & 52.25 & 52.80 & 55.13 & 61.82 & 61.01 \\ 
& \ours  & 54.07 & 56.68 & 59.93 & 62.20 & 64.45 \\\bottomrule
\end{tabular}
    \label{sens:label}
    \vspace{-1em}
\end{table}

\subsection{Sensitivity on Softly-labeled Node Set}
In this subsection, we further analyze how the construction of softly-labeled node set can impact the performance of \ours. We compare the intersection approach in \ours with the following alternations: (1) using pseudo labels from the annotator and build $\V^S_A$ for all unlabeled nodes; (2) using predictions from the teacher network and compile $\V^S_T$ for all unlabeled nodes; and (3)  without adding any soft labels, which is actually \dsgnn. 
The node classification performance of \ours with comparison to the three alternatives is reported in Table \ref{sens:label}. From the table, we make the following observations: (i) Compared with training without soft-labels, i.e., trained on $\mathcal{V}^L$ only, using soft-labels, i.e., $\mathcal{V}_A^S$, $\mathcal{V}_T^S$ and $\mathcal{V}^S$,  can significantly improve the performance, which shows the importance of soft-labels in providing supervision to GNN for classification; and (ii) Though $\mathcal{V}_A^S$, $\mathcal{V}_T^S$ and $\mathcal{V}^S$ all utilize soft-labels, the performance of $\mathcal{V}^S$ is much better than  $\mathcal{V}_A^S$ and $\mathcal{V}_T^S$, which indicates that the teacher network and the annotator may infer some wrongly labeled nodes that could negatively affect the performance. Taking the intersection of these two can help pick nodes with correct soft labels and improve the performance.


\section{Conclusion}
In this paper, we empirically analyze an issue of GNN for semi-supervised node classification, i.e., when labeled nodes are randomly distributed on the graph, nodes with low degrees tend to have very few labeled nodes, which results in  sub-optimal performance on low-degree nodes. To solve this issue, we propose a novel framework \ours, which leverages degree-specific GCN layers and the self-supervised-learning with Bayesian teacher network to introduce more labeled neighbors for low-degree nodes. Experimental results on real-world detests demonstrate the effectiveness of the proposed framework for semi-supervised node classification. Further experiments are conducted to help understand the contributions of each components of \ours.

There are several interesting directions which need further investigation. First, the proposed \dsgnn layer and self-supervised-learning with Bayesian teacher network are generic framework which can benefit various GNNs. In this paper, we only use GCN as backbone. We will investigate the framework for other GNNs such as GAT~\cite{velivckovic2017graph}. Second,  we mainly focus on the degree issue of attributed graphs. Heterogeneous information networks~\cite{shi2016survey} are also pervasive in the real world. Similar issue also exists in heterogeneous graphs. Therefore, we will extend \ours for heterogeneously network by considering different types of links/edges.

\section*{Acknowledgement}
This material is based upon work supported by, or in part by, the National Science Foundation (NSF) under grant  IIS-1909702, IIS-1955851, and the Global Research Outreach program of Samsung Advanced Institute of Technology under grant \#225003. The findings and conclusions in this paper do not necessarily reflect the view of the funding agency.

\bibliographystyle{ACM-Reference-Format}
\bibliography{ref}

\end{document}